\documentclass{article}

\usepackage[numbers,sort&compress,square]{natbib}

\usepackage[main, final]{arxiv_v2}

\usepackage[utf8]{inputenc} %
\usepackage[T1]{fontenc}    %

\usepackage{url}            %
\usepackage{nicefrac}       %

\usepackage{microtype}
\usepackage{amsfonts}
\usepackage{graphicx}
\usepackage{booktabs}
\usepackage{wrapfig}
\usepackage{amsmath}
\usepackage{amsthm}
\usepackage{amssymb}
\usepackage{caption}
\usepackage{algpseudocode}
\usepackage[lined,boxed,commentsnumbered,ruled,longend]{algorithm2e}
\theoremstyle{plain}
\newtheorem{theorem}{Theorem}[section]
\newtheorem{proposition}[theorem]{Proposition}
\newtheorem{lemma}[theorem]{Lemma}

\theoremstyle{definition}

\newtheorem{assumption}[theorem]{Assumption}

\theoremstyle{remark}
\newtheorem{remark}[theorem]{Remark}

\usepackage[subtle, mathdisplays=tight, charwidths=tight, leading=normal]{savetrees}
\usepackage{subcaption} %
\usepackage{booktabs} %
\usepackage{longtable}
\usepackage{textcomp}
\usepackage{enumerate}

\usepackage{parskip}
\usepackage{epigraph}

\usepackage[align=center,shadow=true,shadowsize=4.5pt,nobreak=true,framemethod=tikz,skipabove=9.5pt,skipbelow=9pt,innertopmargin=8pt,innerbottommargin=8pt,innerleftmargin=5pt,innerrightmargin=5pt,leftmargin=1.5pt,rightmargin=1.5pt]{mdframed}
\usetikzlibrary{shadows}

\newlength\savewidth\newcommand\shline{\noalign{\global\savewidth\arrayrulewidth
  \global\arrayrulewidth 1pt}\hline\noalign{\global\arrayrulewidth\savewidth}}

\SetKwInput{KwInput}{Input}
\SetKwInput{KwOutput}{Output}

\usepackage{enumitem}

\usepackage{tcolorbox}
\usepackage[table]{xcolor}
\usepackage{bm}
\usepackage{algpseudocode} %

\usepackage[pagebackref=true,breaklinks=true,colorlinks=true,bookmarks=false]{hyperref}
\definecolor{deepred}{HTML}{940000}
\hypersetup{linkcolor=deepred}
\hypersetup{urlcolor  = [rgb]{0.4,0.15,0.95}}
\hypersetup{citecolor=[rgb]{0.4,0.15,0.95}}
\usepackage[capitalize,noabbrev]{cleveref}

\tcbset{
  myexample/.style={
    colback=white,
    colframe=black!60,
    colbacktitle=black!80,
    coltitle=white,
    fonttitle=\bfseries,
    arc=4mm,
    boxrule=1pt,
    left=3mm, right=3mm, top=1.5mm, bottom=1.5mm,
    title={#1}
  }
}

\definecolor{caseblue}{HTML}{5C6FA8}

\tcbset{
  casestudy/.style={
    colback=white,
    colframe=black!60,
    colbacktitle=caseblue,    %
    coltitle=white,
    fonttitle=\bfseries,
    arc=4mm,
    boxrule=1pt,
    left=3mm, right=3mm, top=1.5mm, bottom=1.5mm,
    title={#1}
  }
}
\usepackage{multirow}
\usepackage{makecell}

\usepackage{tikz} %

\newcommand{\ie}{\emph{i.e.}}
\newcommand{\eg}{\emph{e.g.}}

\usepackage[toc,page,header]{appendix}
\usepackage{minitoc}

\renewcommand \thepart{}
\renewcommand \partname{}

\newcommand\blfootnote[1]{%
  \begingroup
  \renewcommand\thefootnote{}\footnote{#1}%
  \addtocounter{footnote}{-1}%
  \endgroup
}

\usepackage{comment}

\title{\underline{Pion}: A Spectrum-Preserving Optimizer\\via Orthogonal Equivalence Transformation}

\author{\fontsize{8.75pt}{\baselineskip}\selectfont
  Kexuan Shi\textsuperscript{1,\textdagger}~~~~Hanxuan Li\textsuperscript{1,\textdagger}~~~~Zeju Qiu\textsuperscript{1,2}~~~~Yandong Wen\textsuperscript{3}~~~~Simon Buchholz\textsuperscript{2}~~~~Weiyang Liu\textsuperscript{1,*}\\[0.6mm]\fontsize{8.75pt}{\baselineskip}\selectfont
  \textsuperscript{1}The Chinese University of Hong Kong~~~~\textsuperscript{2}Max Planck Institute for Intelligent Systems~~~~\textsuperscript{3}Westlake University
  }

\begin{document}

\maketitle
\blfootnote{\textsuperscript{\textdagger}Equal contribution \quad~ \textsuperscript{*}Corresponding author \quad~ Project page:~\href{https://spherelab.ai/pion}{\tt spherelab.ai/pion}}

\doparttoc
\faketableofcontents

\vspace{-8.5mm}
\begin{abstract}
\vspace{-.8mm}
  We introduce Pion, a spectrum-preserving optimizer for large language model (LLM) training based on orthogonal equivalence transformation. Unlike additive optimizers such as Adam and Muon, Pion updates each weight matrix through left and right orthogonal transformations, preserving its singular values throughout training. This yields an optimization mechanism that modulates the geometry of weight matrices while keeping their spectral norm fixed. We derive the Pion update rule, systematically examine its design choices, and analyze its convergence behavior along with several key properties. Empirical results show that Pion offers a stable and competitive alternative to standard optimizers for both LLM pretraining and finetuning.
\end{abstract}

\section{Introduction}

As large language models (LLMs) continue to scale, the difficulty of training them also increases significantly. One of the most critical challenges today is designing optimizers that are both efficient and stable. Training stability can be partially characterized by the Maximal Update Parameterization ($\mu$P)~\cite{yang2023spectral}, where spectral norms of weights and updates are constrained such that width-invariant activations are of constant scale and hence prevent explosions.  By performing the steepest descent under the spectral norm through update orthogonalization, Muon~\cite{jordan2024muon} has emerged as a competitive alternative to AdamW~\cite{kingma2014adam,loshchilov2019decoupled}. Although Muon's orthogonalization ensures that each update is easily $\mu$P-compatible, the spectral norms of the weight matrices themselves may still drift throughout training. Built upon Muon, recent work addresses this issue either by introducing normalization~\cite{ding2021cogview,henry2020query,team2025kimi,li2025normuon} or by incorporating spectral retraction directly into the update rule~\cite{xie2026controlled}. Rather than adapting Muon to achieve $\mu$P, we introduce Pion, a fundamentally different optimizer that constrains the spectral norm of both weights and updates through its optimization dynamics. Specifically, Pion is derived from orthogonal equivalence transformations of weight matrices, updating each matrix via coupled left and right orthogonal transformations. Its design is guided by the following principles:

\vspace{0.1mm}
\begin{itemize}[leftmargin=*,nosep]
\setlength\itemsep{0.4em}
    \item \textbf{Algorithmic spectrum control}: Pion derives the update rule directly on the iso-spectral manifold, eliminating the need for explicit normalization while preserving the weight spectrum throughout optimization. This property is particularly desirable, as the upper-bounded spectral norm of the weight is closely linked to stronger generalization~\cite{miyato2018spectral,yoshida2017spectral,jiang2019computation,bartlett2017spectrally}. Moreover, the update’s spectral norm is also guaranteed to be upper bounded, making Pion easily compatible with $\mu$P.
    \item \textbf{Minimum energy training}: Pion updates weight matrices via orthogonal equivalence transformations, which inherently preserve hyperspherical energy~\cite{liu2018mhe,Liu2021SphereUni}. This energy quantifies how uniformly normalized neurons are distributed on the hypersphere, and lower energy has been shown to correlate with better generalization~\cite{liu2018mhe,Lin20CoMHE,Liu2021OPT}. Because zero-mean Gaussian weight initialization yields a minimum-energy configuration, Pion provably preserves this configuration throughout training, maintaining a uniform hyperspherical distribution of normalized neurons.
\end{itemize}

Pion is inspired by orthogonal pretraining~\cite{Liu2021OPT} and POET~\cite{qiu2025reparameterized,qiu2026poetx}, which reparameterizes each weight matrix as a left orthogonal matrix, a randomly initialized base weight, and a right orthogonal matrix, learning only the two orthogonal factors. This reparameterization enforces spectrum preservation by construction, but recasts the optimization variables from the weights themselves to auxiliary orthogonal parameters. While this shift enables greater memory efficiency~\cite{qiu2026poetx}, it also complicates training dynamics, giving rise to issues such as loss spikes and the need for careful momentum design. Pion, short for \underline{P}OET-\underline{i}nduced \underline{o}ptimizer with \underline{n}o reparameterization, removes this auxiliary parameterization and instead turns the same principle into a direct optimizer. Specifically, Pion updates each weight matrix through left and right orthogonal transformations, preserving its singular-value spectrum throughout training while operating directly on the model weights. This yields a spectrum-preserving optimization dynamics that retains the geometric inductive bias of POET in a simpler and more stable optimizer form.

\section{\underline{Pion}: A Spectrum-Preserving Optimizer for LLM Training}
\subsection{Background and Preliminaries}

\begin{wrapfigure}{r}{0.45\linewidth}
\scriptsize
\centering
\vspace{-2.6mm}
\includegraphics[width=1.0\linewidth]{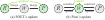}
\vspace{-2em}
\caption{\scriptsize Comparison of POET and Pion (Green: learnable).}
    \label{fig:poetpion}
\vspace{-.75mm}
\end{wrapfigure}

POET~\cite{qiu2025reparameterized,qiu2026poetx} reparameterizes each weight matrix as $\bm{W}_{RP}=\bm{R}\bm{W}_0\bm{P}$, where $\bm{W}_0\in\mathbb{R}^{d_{\text{out}}\times d_{\text{in}}}$ is fixed at random initialization, and $\bm{R}\in\mathbb{R}^{d_{\text{out}}\times d_{\text{out}}}$, $\bm{P}\in\mathbb{R}^{d_{\text{in}}\times d_{\text{in}}}$ are trainable orthogonal matrices. This corresponds to an orthogonal equivalence transformation that acts on both sides of $\bm{W}_0$, yielding the forward pass weight $\bm{R}\bm{W}_0\bm{P}$.
After training, $\bm{R}$ and $\bm{P}$ are merged into the weight matrix, so POET incurs no additional inference overhead. However, since optimization is performed over two orthogonal matrices while the weight matrix remains fixed throughout training, this reparameterization poses non-trivial challenges in both training stability and cross-architecture adaptability. Motivated by this, Pion introduces a novel update rule that fully preserves the weight spectrum without reparameterization.

\subsection{Spectrum-Preserving Update Rule}
We begin with the intuition behind Pion's update rule. Consider a general weight matrix $\bm{W}\in \mathbb{R}^{d_{\text{out}} \times d_{\text{in}}}$. At iteration $t$, we can trivially write $\bm{W}_t$ as 
$\bm{W}_t = \bm{I}_{d_{\text{out}}} \, \bm{W}_t \, \bm{I}_{d_{\text{in}}}$
where $\bm{I}_{d_{\text{out}}}$ and $\bm{I}_{d_{\text{in}}}$ are identity matrices of the corresponding sizes. 
Geometrically, each identity matrix is the neutral element of an orthogonal group, representing a zero rotation.
Pion leverages this observation by \emph{updating the identity factors directly on the orthogonal group} without an explicit reparameterization like $\bm{R}\bm{W}\bm{P}$.
This induces left and right orthogonal transformations of $\bm{W}_t$, preserving its spectrum. See the comparison in Figure~\ref{fig:poetpion}.

The challenge is to update the identity factors on the orthogonal group. Since the orthogonal group is a compact Lie group, we use standard techniques from Lie group optimization~\citep{mhammedi2017efficient,lezcano2019cheap,boumal2023introduction}.
Let $\bm{G}_t \in \mathbb{R}^{d_{\text{out}}\times d_{\text{in}}}$ denote the gradient of $\bm{W}_t$. Pion updates $\bm{W}_t$ in a spectrum-preserving manner:

\vspace{-4mm}
\begin{equation}\label{eq:update}
\footnotesize
    \bm{G}^{\text{in}}_t = \bm{W}_t^\top \bm{G}_t - (\bm{W}_t^\top \bm{G}_t)^\top,\quad 
    \bm{G}^{\text{out}}_t = \bm{G}_t \bm{W}_t^\top - (\bm{G}_t \bm{W}_t^\top)^\top,\quad %
    \boxed{\bm{W}_{t+1} = \exp(-\eta \bm{G}^{\text{out}}_t)\, \bm{W}_t\, \exp(-\eta \bm{G}^{\text{in}}_t)}
\end{equation}
\vspace{-3.5mm}

where $\exp(\cdot)$ denotes the matrix exponential. The spectrum-preserving update rule can be understood in three steps; its full derivation is given in Appendix~\ref{app:eq1-derivation}.
First, we apply the chain rule to the two identity factors $\bm{I}_{d_{\text{out}}}$ and $\bm{I}_{d_{\text{in}}}$, computing the corresponding gradients $\bm{G}_t \bm{W}_t^\top$ and $\bm{W}_t^\top \bm{G}_t$.
Second, we enforce skew-symmetry to project these gradients onto the Lie algebra, producing Lie algebra elements from the two identity factors. %
Finally, we map these elements back to the Lie group via the matrix exponential, which produces valid orthogonal transformations.

\subsection{Properties of Pion's Update Rule}

Pion's update rule can be written as $\bm{W}_{t+1}=\bm{R}_t\bm{W}_t\bm{P}_t$ with $\bm{R}_t=\exp(-\eta\bm{G}^{\text{out}}_t)$ and $\bm{P}_t=\exp(-\eta\bm{G}^{\text{in}}_t)$, where both $\bm{R}_t$ and $\bm{P}_t$ are orthogonal. Hence Pion transforms the row and column subspaces of $\bm{W}_t$ while preserving its singular values. We start with Pion's geometric structures.

\begin{proposition}[Geometric structure of Pion's update]
\label{prop:total_rotation}
\label{prop:bilateral_rotation}
Let $\Delta\bm{W}=\bm{W}_{t+1}-\bm{W}_t$. Pion's update preserves the singular values of $\bm{W}_t$ and only changes its row and column subspaces through orthogonal transformations, in which positive determinant leads to rotation. Consequently, $\|\Delta\bm{W}\|_F$ characterizes the total rotational strength applied to $\bm{W}_t$. At a finer level, the in-side and out-side updates decompose into independent planar rotations on orthogonal $2$D invariant subspaces induced by $\bm{G}^{\text{in}}_t$ and $\bm{G}^{\text{out}}_t$. The quantities $\frac{1}{\sqrt{d_{\text{in}}}}\|\bm{G}^{\text{in}}_t\|_F$ and $\frac{1}{\sqrt{d_{\text{out}}}}\|\bm{G}^{\text{out}}_t\|_F$ characterize the average rotation magnitudes on the two sides, while $\|\bm{G}^{\text{in}}_t\|_2$ and $\|\bm{G}^{\text{out}}_t\|_2$ control the maximum rotation angles.
\end{proposition}

Because $\bm{R}_t$ and $\bm{P}_t$ are orthogonal, they preserve the $\ell_2$ norms of the rows and columns of $\bm{W}_t$. Hence, $\Delta\bm{W}$ reflects angular deviation rather than rescaling. The update norm is therefore directly interpretable as rotational motion, unlike vanilla gradient descent, which generally entangles changes in magnitude and direction. Appendix~\ref{app:pion_geometry_derivation} provides a detailed derivation of the planar-rotation view.

The following theorem shows that Pion's spectrum-preserving update admits first-order convergence guarantees on the isospectral manifold under standard assumptions. The proof is given in Appendix~\ref{app:convergence}.

\begin{theorem}[Pion's Convergence on the Isospectral Manifold]
\label{thm:bilateral-pion-convergence}
Assume that $f(\bm{W})$ is $L$-smooth and lower bounded by $f_{\inf}$. Let the stochastic gradient be $\tilde{\bm{G}}_t=\nabla f(\bm{W}_t)+\bm{\xi}_t$, where $\mathbb{E}_t[\bm{\xi}_t]=0$ and $\mathbb{E}_t[\|\bm{\xi}_t\|_F^2]\le\sigma^2$. Assume the iterates remain on the iso-spectral manifold induced by $\bm{W}_0$, so that $\|\bm{W}_t\|_2=\gamma$ for all $t$. Define
$
\bm{G}_t^{\mathrm{in}}=\bm{W}_t^\top \nabla f(\bm{W}_t)-\nabla f(\bm{W}_t)^\top \bm{W}_t
$
and
$
\bm{G}_t^{\mathrm{out}}=\nabla f(\bm{W}_t)\bm{W}_t^\top-\bm{W}_t\nabla f(\bm{W}_t)^\top .
$
Assume the updates are performed for $T+1$ iterations with step size $\eta=C/\sqrt{T+1}$, where $C>0$ is sufficiently small such that the one-step descent coefficient is positive. Then we have

\vspace{-4.5mm}
\begin{equation}
\footnotesize
    \min_{0\le t\le T}
    \mathbb{E}
    \left[
    \|\bm{G}_t^{\mathrm{in}}\|_F^2
    +
    \|\bm{G}_t^{\mathrm{out}}\|_F^2
    \right]
    \le
    \frac{1}{\sqrt{T+1}}
    \left(C_1+C_2\right)= \mathcal{O}(\frac{1}{\sqrt{T}}),
\end{equation}
\vspace{-3.5mm}

where $C_1$ depends on the initial optimality gap $f(\bm{W}_0)-f_{\inf}$, and $C_2$ depends on $L$, $\gamma$, $\sigma^2$, and $C$.
\end{theorem}

\begin{remark}
Theorem~\ref{thm:bilateral-pion-convergence} establishes first-order stationarity on the isospectral manifold. However, it does not guarantee stationarity in the ambient space, as a critical point of the restricted problem may still admit a nonzero ambient gradient that is normal to the manifold. Appendix~\ref{app:ambient-gradient} gives more details.
\end{remark}

\subsection{Design Principles for Stable Training and Convergence}
\label{sec:design_exploration_protocol}

While Pion's update rule offers a simple and functional approach to spectrum-preserving optimization, training practical LLMs demands additional design choices for greater stability. To this end, we explore the following design principles. We note that our exploration is by no means comprehensive, but rather represents an initial yet principled effort toward building a stable spectrum-preserving optimizer.
For rapid prototyping, we perform all design explorations using a 60M-parameter LLaMA-based model~\citep{zhang2019root,shazeer2020glu,touvron2023llama}, a common setup for ablation~\citep{zhao2024galore,qiu2025reparameterized,gu2026mano}. All the models in this section are trained on C4~\citep{raffel2020exploring} with sequence length 256 for 9.6B tokens, ensuring sufficient training.

\subsubsection{Consistent Update}
\label{sec:consistent_rotational_updates}

\begin{wrapfigure}{r}{0.375\linewidth}
\scriptsize
\centering
\vspace{-1mm}
\includegraphics[width=1.0\linewidth]{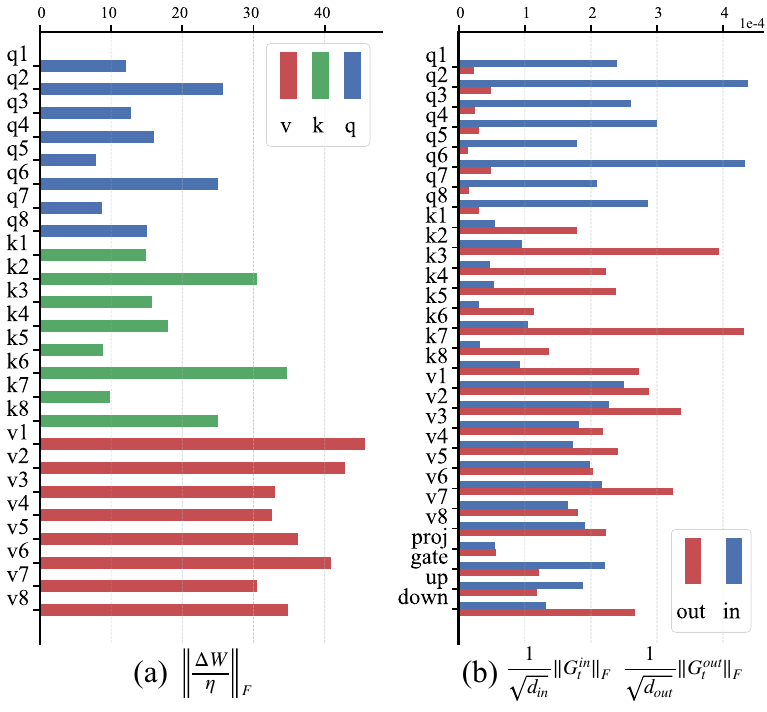}
\vspace{-2.25em}
\caption{\scriptsize Inconsistent updates in Pion.}
    \label{fig:update_with_blocks_and_rotations}
\vspace{1mm}
\end{wrapfigure}

To train deep neural networks effectively, prior work~\citep{ioffe2015batch,ba2016layer,xiong2020layer,glorot2010understanding,yang2023spectral} has sought to keep network components operating under stable input/output distributions and receiving consistent feature updates. This consistency principle has also guided the scaling of modern optimizers~\citep{liu2025muon,xie2026controlled,bernstein2024old} to large models. In particular, optimizer-induced parameter updates $\frac{1}{\eta}\Delta \bm{W} $ are expected to be scale-consistent, \ie, their norms should grow proportionally with the size of the corresponding parameter space. We analyze the training dynamics of Pion and identify two notable violations of this principle. First, the original update produces substantial heterogeneity in the normalized update magnitude, $\|\frac{1}{\eta}\Delta \bm{W}\|_F$, across identically sized parameter matrices within the same layer, as shown in Figure~\ref{fig:update_with_blocks_and_rotations}(a).
Second, for individual parameter matrices, the average bilateral rotation angles, which are measured by $\frac{1}{\sqrt{d_{\text{in}}}}\|\bm{G}^{\text{in}}_t\|_F$ and $\frac{1}{\sqrt{d_{\text{out}}}}\|\bm{G}^{\text{out}}_t\|_F$, show a pronounced imbalance between the input and output sides, as shown in Figure~\ref{fig:update_with_blocks_and_rotations}(b).
These inconsistencies stem from a geometric mismatch: the naive update's transformation dynamics are neither scale-consistent across the full parameter feature space nor balanced between each matrix's input and output feature spaces.
We resolve these mismatches by controlling update magnitude in the Lie algebra. Specifically, we normalize the in-side and out-side skew-symmetric gradients:

\vspace{-4mm}
\begin{equation}
\footnotesize
\bm{G}^{\text{out}}_t \leftarrow \sqrt{d_{\text{out}}} \cdot \bm{G}^{\text{out}}_t / \|\bm{G}^{\text{out}}_t\|_F,\quad
\bm{G}^{\text{in}}_t \leftarrow \sqrt{d_{\text{in}}} \cdot \bm{G}^{\text{in}}_t / \|\bm{G}^{\text{in}}_t\|_F .
\label{eq:brb}
\end{equation}
\vspace{-4mm}

To enforce scale consistency across weight matrices, we introduce a per-weight coefficient $\alpha_t$ below:

\vspace{-4mm}
\begin{equation}
\footnotesize
\bm{W}_{t+1}
= \exp(-\eta\alpha_t\bm{G}^{\text{out}}_t)\,\bm{W}_t\,\exp(-\eta\alpha_t\bm{G}^{\text{in}}_t),
\quad
\text{s.t. } \mathrm{RMS}(\Delta\bm{W}/\eta)\approx\mathrm{const}.
\label{eq:alpha}
\end{equation}
\vspace{-4mm}

Directly computing $\frac{1}{\eta}\Delta\bm{W}$ under the exponential map would nearly double the computational cost, so we instead use a first-order approximation to compute $\alpha$:

\vspace{-4mm}
\begin{equation}
\footnotesize
    \alpha_t \approx
    \frac{
    c\sqrt{d_{\text{out}}d_{\text{in}}}
    }
    {
    \|\Delta \bm{W}/\eta \|_F+\epsilon
    },\quad \text{where}~~\Delta\bm{W}/\eta\approx-\bm{G}^{\text{out}}_t\bm{W}_t-\bm{W}_t\bm{G}^{\text{in}}_t.
\end{equation}
\vspace{-3mm}

Such per-weight scaling in Equation~\eqref{eq:alpha} makes the effective rotational update magnitude scale with the size of each weight matrix, so that the rotational strengths across different matrices are approximately proportional to the square root of the ratio of their parameter counts. Such RMS scaling is also used to ensure proportionally consistent updates in Euclidean space~\citep{liu2025muon,gu2026mano}.

\begin{wrapfigure}{r}{0.31\linewidth}
\scriptsize
\centering
\vspace{-4mm}
\includegraphics[width=.99\linewidth]{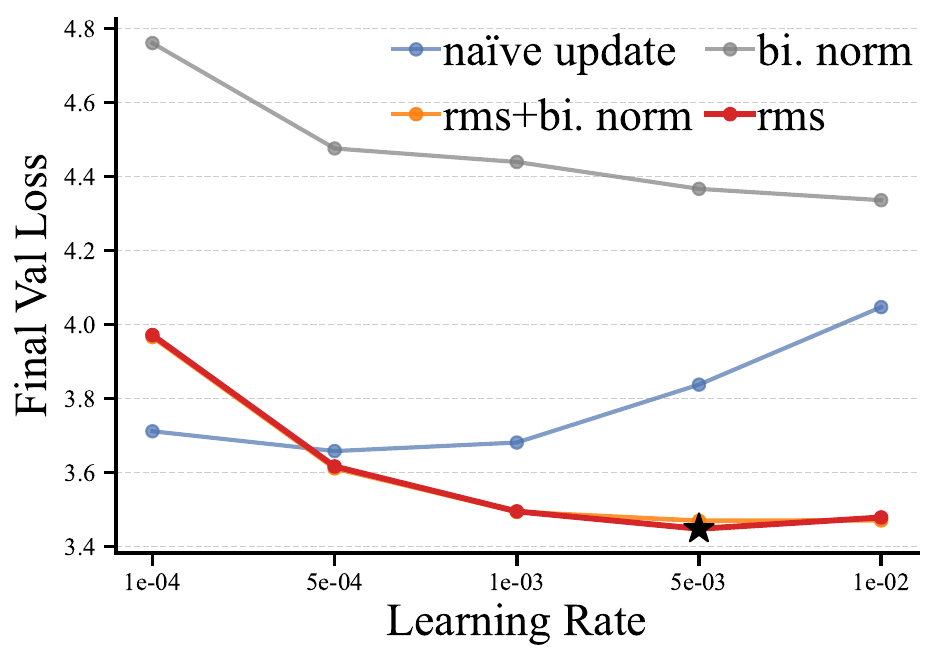}
\vspace{-5.5mm}
\caption{\scriptsize Validation loss comparison of different consistent update strategies. ``$\bigstar$'' denotes the achieved minimum validation loss.}
\label{fig:loss_of_consistent_update}
\vspace{1mm}
\end{wrapfigure}

Results in Figure~\ref{fig:loss_of_consistent_update} show that the naive update rule performs well only under small learning rates and diverges when the learning rate becomes large. In contrast, the RMS-controlled scale-consistent update ensures consistent update magnitudes across matrices and can effectively utilize larger learning rates to accelerate convergence. Applying bilateral normalization alone, or combining RMS control with bilateral normalization, does not further improve training stability or final performance. Taking a step further, these results suggest that scale-consistent rotation update across parameter matrices is key to stable spectrum-preserving optimization. Due to its empirical effectiveness, we therefore adopt RMS-controlled scale consistency as a core component of Pion. Because the bilateral normalization does not show obvious benefits, it is not adopted in Pion.

\begin{figure}[t!]
    \centering
    \setlength{\abovecaptionskip}{3pt}
    \setlength{\belowcaptionskip}{-6pt}
    \vspace{-1mm}
    \includegraphics[width=\linewidth]{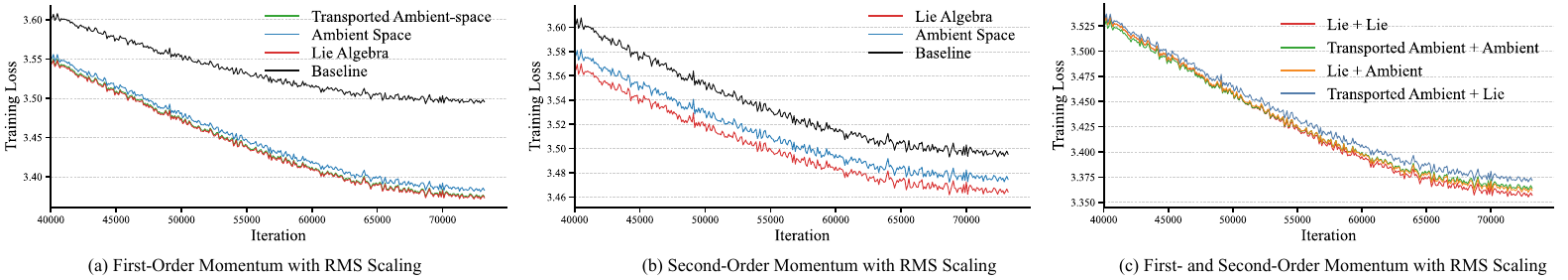}
    \caption{\scriptsize Training loss curves of momentum designs. Figures (a) and (b) show first-order-only and second-order-only momentum, both with RMS scaling. Figure (c) combines the two momentum techniques. ``Lie+Lie'' denotes Lie-algebra first- and second-order momentum. ``Transported Ambient + Ambient'' denotes transported ambient-space first-order momentum with ambient-space second-order momentum.}
    \label{fig:momentum_loss_curves}
\end{figure}

\subsubsection{Momentum Design}
\label{sec:momentum_geometry}

A key ingredient to accelerate gradient-based iterative optimization is momentum~\citep{polyak1964some,Nesterov1983AMF}, which uses accumulated gradient information to adapt the current update direction. This technique has inspired a range of highly effective first-order optimizers for deep learning~\citep{duchi2011adaptive,zeiler2012adadelta,kingma2014adam,reddi2019convergence}. In this section, we explore how to integrate momentum into the Pion's update rule in Equation~\eqref{eq:update}.

\textbf{Transported ambient-space first-order momentum}. We note that the update trajectory $\{\bm{W}_t\}_{t=1}^{\infty}$ evolves on a smooth iso-spectral manifold $\mathcal{M}$ with nonzero curvature (we assume the singular values of $\bm{W}_0$ are distinct and nonzero). Hence, momentum should be expressed in a consistent tangent space. A natural approach is to parallel-transport the historical momentum to the current tangent space~\citep{smith2014optimization,becigneul2018riemannian}. 
Following this idea, we derive a transported first-order momentum update below. After the $t$-th step, $\exp(-\eta \bm{G}^{\mathrm{out}}_t)$ and $\exp(-\eta \bm{G}^{\mathrm{in}}_t)$ used in the update are reused to transport $\bm{m}_t$ to the tangent space:

\vspace{-4mm}
\begin{equation}
\footnotesize
\begin{aligned}
    &\bm{m}_t = \beta_1 \bm{m}_{t-1} + (1-\beta_1) \bm{G}_t,\quad \bm{G}^{\text{in}}_t = \bm{W}_t^\top \bm{m}_t - (\bm{W}_t^\top \bm{m}_t)^\top,\quad 
    \bm{G}^{\text{out}}_t = \bm{m}_t \bm{W}_t^\top - (\bm{m}_t \bm{W}_t^\top)^\top, \\
    &\quad\quad~~ \underbrace{\bm{W}_{t+1} = \exp(-\eta \bm{G}^{\text{out}}_t)\, \bm{W}_t\, \exp(-\eta \bm{G}^{\text{in}}_t)}_{\text{Update rule with first-order mementum}}, \quad \quad \underbrace{\bm{m}_t \leftarrow \exp(-\eta \bm{G}^{\text{out}}_t)\, \bm{m}_t\, \exp(-\eta \bm{G}^{\text{in}}_t)}_{ \text{Transported to}~T_{\bm{W}_{t+1}} \mathcal{M}}.
\end{aligned}
\end{equation}
\vspace{-3mm}

where $T_{\bm{W}_{t+1}}\mathcal{M}$ denote the tangent space at $\bm{W}_{t+1} \in \mathcal{M}$. The next gradient $\bm{G}_{t+1}$ is also in $T_{W_{t+1}}\mathcal{M}$.
This parallel transport improves the accuracy of first-order gradient estimation. For comparison, we also consider a first-order momentum variant without transport, referred to as ambient-space momentum. 
As shown in Figure~\ref{fig:momentum_loss_curves}(a), we empirically compare the new update rule above with or without momemtum transportation, and the transported version consistently achieves faster convergence. 

\textbf{Lie algebra first-order momentum}. Because Pion's update is performed in the Lie algebra, \ie, within a single tangent space. This property provides another way to accumulate first-order momentum, in which the accumulation is performed directly in the Lie algebra.
Let $\bm{m}^{\text{in}}_t \in \mathbb{R}^{d_{\text{in}}\times d_{\text{in}}}$ and $\bm{m}^{\text{out}}_t \in \mathbb{R}^{d_{\text{out}}\times d_{\text{out}}}$ denote the momentum variables associated with the in-side and out-side gradient update, respectively. Given the same $\bm{G}^{\text{in}}_t$ and $\bm{G}^{\text{out}}_t$ in Equation~\eqref{eq:update}, the modified update rule becomes

\vspace{-4mm}
\begin{equation}\label{eq:lie_momentum}
\footnotesize
    \underbrace{\bm{m}^{\text{in}}_t = \beta_1 \bm{m}^{\text{in}}_{t-1} + (1-\beta_1) \bm{G}^{\text{in}}_t}_{\text{Input-side Lie algebra momentum}}, \quad 
    \underbrace{\bm{m}^{\text{out}}_t = \beta_1 \bm{m}^{\text{out}}_{t-1} + (1-\beta_1) \bm{G}^{\text{out}}_t}_{\text{Output-side Lie algebra momentum}},\quad \underbrace{\bm{W}_{t+1} = \exp(-\eta \bm{m}^{\text{out}}_t)\, \bm{W}_t\, \exp(-\eta \bm{m}^{\text{in}}_t)}_{\text{Update rule with first-order momentum}}.
\end{equation}
\vspace{-2.5mm}

As shown in Figure~\ref{fig:momentum_loss_curves}, the Lie algebra momentum achieves the fastest convergence, slightly outperforming the transported ambient-space momentum. All the first-order momentum formulations are spectrum-preserving by construction, as they generate updates via skew-symmetric operators followed by the matrix exponential map. Ambient-space momentum is the most efficient in compute and memory, but produces biased estimates due to tangent space mismatch. Transported ambient-space momentum corrects this bias via parallel transport at added computational cost, with no extra memory overhead. Lie algebra momentum achieves exact, geometrically consistent estimation at comparable compute cost, but requires additional memory for separate in- and out-side variables.

\textbf{Second-order momentum}. Second-order momentum tracks the running average of squared gradients, serving as an adaptive normalization factor. Unlike its first-order counterpart, it accumulates no directional information across tangent spaces and therefore does not require parallel transport in manifold optimization~\citep{becigneul2018riemannian}. Guided by this observation and the design principles established for first-order momentum, we consider two natural implementations of second-order momentum: (1) estimating second-order momentum in the ambient space using the full gradient $\bm{G}_t$; and (2) modeling second-order statistics separately for the in-side and out-side gradients. Specifically, we use the standard second-order momentum formulation in AdamW. For example, the second-order momentum for the in-side update in Equation~\eqref{eq:lie_momentum} is $\bm{v}^{\mathrm{in}}_t = \beta_2 \bm{v}^{\mathrm{in}}_{t-1} + (1-\beta_2)(\bm{G}^{\mathrm{in}}_t \odot \bm{G}^{\mathrm{in}}_t)$ where $\odot$ is the element-wise multiplication. The complete algorithms are given in Algorithm~\ref{alg:pion} and Appendix~\ref{app:imp_details}.

From Figure~\ref{fig:momentum_loss_curves}(b), we observe that the Lie Algebra variants consistently outperform the ambient-space variants. We then examine their combined behavior in Figure~\ref{fig:momentum_loss_curves}(c). The Lie+Lie variant performs best, consistent with the trends observed for each momentum order individually. In contrast, mixed variants perform slightly worse, suggesting that mismatched deployments hinder this complementarity. This observation further suggests that, under our spectrum-preserving update rule, accumulating momentum in the Lie algebra provides a more natural and effective formulation. Thus, we adopt the Lie+Lie and Transported-Ambient+Ambient variants as two functional implementations of momentum in Pion, representing two effective design choices identified from our exploration.

\subsubsection{Alternate Update}
\label{sec:alternate_update_principle}

The bilateral update in Equation~\eqref{eq:update} applies both input-side and output-side orthogonal transformations at every iteration. We propose a more computationally efficient Pion variant that alternates between the input-side and output-side updates across successive iterations:

\vspace{-4mm}
\begin{equation}\label{eq:alternating_update}
\footnotesize
    \bm{W}_{t+1} =
        \exp(-\eta(1-\psi)
    \bm{G}^{\text{out}}_t) \cdot \bm{W}_t \cdot\exp(-\eta \psi\bm{G}^{\text{in}}_t),\quad\text{with}~\psi=1~\text{if}~t~\text{is odd},~~\psi=0~\text{if}~t~\text{is even}.
\end{equation}
\vspace{-4mm}

\begin{wrapfigure}{r}{0.42\linewidth}
\scriptsize
\centering
\vspace{-0.5mm}
\includegraphics[width=1.0\linewidth]{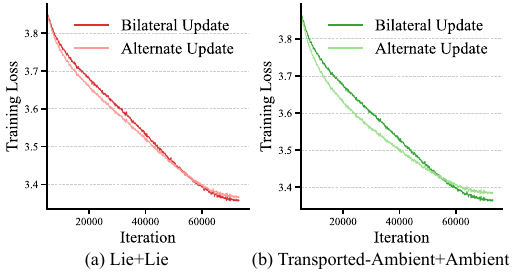}
\vspace{-5mm}
\caption{\scriptsize Training loss of bilateral and alternate update.}
\label{fig:alternate_update}
\end{wrapfigure}

The alternate update remains spectrum-preserving and also decouples the two orthogonal transformations across iterations, largely reducing per-step computation. 
Moreover, the alternation can be naturally extended to occur every few steps rather than every step.
From Figure~\ref{fig:alternate_update}, we observe that alternate update achieves performance close to bilateral update for both Lie+Lie and Transported-Ambient+Ambient variants. For Lie+Lie, the final loss of alternate update is 3.3654, only about $0.23\%$ higher than the 3.3575 achieved by bilateral update. 
This small gap suggests that updating the two transformations simultaneously is not always necessary to obtain strong optimization performance, and that much of the benefit can be retained through a more lightweight alternating scheme.
Alternate updates are slightly faster early in training, while bilateral updates overtake them near convergence, reflecting a tradeoff between efficiency and refinement. Overall, the alternate update is a strong compromise, as it preserves spectral structure, lowers computational cost, and achieves nearly the same final performance as the more expensive bilateral update.

\begin{figure}[h!]
    \centering
    \setlength{\abovecaptionskip}{3pt}
    \setlength{\belowcaptionskip}{8pt}
    \vspace{2.5mm}
    \includegraphics[width=\linewidth]{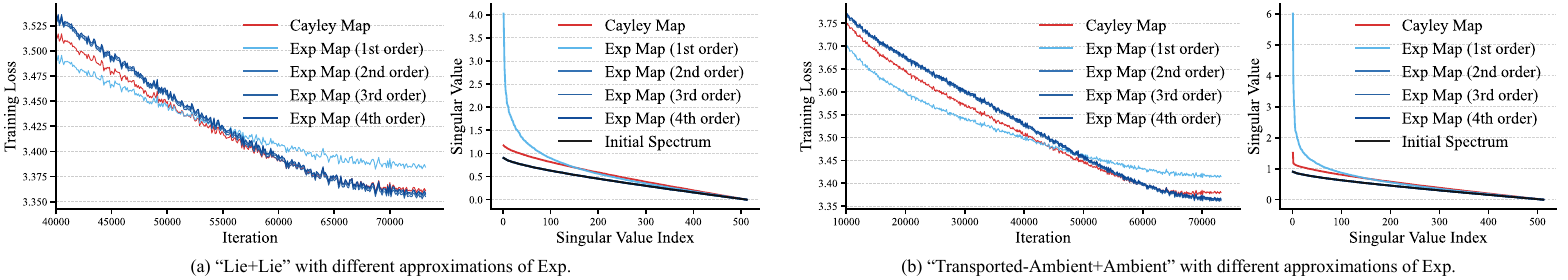}
    \caption{\scriptsize Comparison of different approximation schemes for $\mathrm{Exp}(\cdot)$. Left panels: training loss curves under each scheme for the Lie+Lie and Transported-Ambient+Ambient configurations. Right panels: final singular value spectra alongside the initial spectra for reference.}
    \label{fig:exp_approx}
\end{figure}

\subsubsection{Efficient Approximation to Matrix Exponential Map}
\label{sec:exp_approx_principle}

All Pion variants require computing the matrix exponential. Since exact evaluation of $\exp(\bm{A})$ is generally expensive, we consider two efficient approximations. The first is the Cayley transform~\citep{helfrich2018orthogonal,li2020efficient,Liu2021OPT}:
$\exp(\bm{A}) \approx \left(\bm{I}-\tfrac12\bm{A}\right)^{-1}\left(\bm{I}+\tfrac12\bm{A}\right)$, which agrees with $\exp(\bm{A})$ up to second order with error $\mathcal{O}(\|\bm{A}\|^3)$ and strictly preserves orthogonality for skew-symmetric $\bm{A}$. In practice, the matrix inverse can be further cheapened via low-order series expansions~\citep{qiu2023controlling,qiu2025reparameterized}.
The second is a truncated power series: $\exp(\bm{A})\approx \sum_{k=0}^{L}\frac{\bm{A}^k}{k!}$, whose truncation error decays rapidly with $L$ when $\|\bm{A}\|_F$ is small.

Figure~\ref{fig:exp_approx} compares first- to fourth-order power-series approximations with the Cayley transform.
The first-order approximation degrades both convergence and singular-value preservation, while the Cayley variant offers only modest gains. The second-order approximation preserves the spectrum nearly as well as higher-order variants.
Unlike conventional orthogonal-group optimization where errors in $\exp(\bm{A}_t)$ can accumulate through updates, Pion's update always starts from the identity matrix (Equation~\ref{eq:update}). This prevents repeated error compounding and improves numerical robustness, making higher-order approximations unnecessary. Pion therefore adopts a second-order exponential approximation.

\subsection{Detailed Implementation and Computational Complexity}
\label{sec:distilled_pion_optimizer}

\begin{wrapfigure}{r}{0.43\textwidth}
\vspace{-.75mm}
\begin{minipage}{0.43\textwidth}
\SetAlCapFnt{\scriptsize}
\SetAlCapNameFnt{\scriptsize}
\begin{algorithm}[H]
\caption{The Pion Optimizer (Lie Algebra)}
\label{alg:pion}
\scriptsize
\KwIn{Learning rate $\eta$, momentum coefficients $\beta_1,\beta_2$, RMS constant $c$, stability constant $\epsilon$, alternating flag, initial weight matrix $\bm{W}_0 \in \mathbb{R}^{d_{\mathrm{out}}\times d_{\mathrm{in}}}$}
\KwOut{Optimized parameter $\bm{W}_t$}
Initialize $\bm{m}^{\mathrm{in}}_0, \bm{v}^{\mathrm{in}}_0 \leftarrow \bm{0} \in \mathbb{R}^{d_{\mathrm{in}}\times d_{\mathrm{in}}}$,
$\bm{m}^{\mathrm{out}}_0, \bm{v}^{\mathrm{out}}_0 \leftarrow \bm{0} \in \mathbb{R}^{d_{\mathrm{out}}\times d_{\mathrm{out}}}$\;
Define $\mathcal{E}_2(\bm{A}, \alpha) \leftarrow \bm{I}+\eta\alpha \bm{A}+\frac{1}{2}(\eta\alpha \bm{A})^2$\;
\For{$t=1,2,\ldots$}{
    $\bm{G}_t \leftarrow \nabla_{\bm{W}} f(\bm{W}_{t-1})$\;
    $\bm{G}^{\mathrm{in}}_t \leftarrow \bm{W}_{t-1}^{\top}\bm{G}_t - \bm{G}_t^{\top}\bm{W}_{t-1}$, 
    $\bm{G}^{\mathrm{out}}_t \leftarrow \bm{G}_t\bm{W}_{t-1}^{\top} - \bm{W}_{t-1}\bm{G}_t^{\top}$\;
    $\bm{m}^{\mathrm{in}}_t \leftarrow \beta_1 \bm{m}^{\mathrm{in}}_{t-1} + (1-\beta_1)\bm{G}^{\mathrm{in}}_t$, 
    $\bm{m}^{\mathrm{out}}_t \leftarrow \beta_1 \bm{m}^{\mathrm{out}}_{t-1} + (1-\beta_1)\bm{G}^{\mathrm{out}}_t$\;
    $\bm{v}^{\mathrm{in}}_t \leftarrow \beta_2 \bm{v}^{\mathrm{in}}_{t-1} + (1-\beta_2)(\bm{G}^{\mathrm{in}}_t \odot \bm{G}^{\mathrm{in}}_t)$, 
    $\bm{v}^{\mathrm{out}}_t \leftarrow \beta_2 \bm{v}^{\mathrm{out}}_{t-1} + (1-\beta_2)(\bm{G}^{\mathrm{out}}_t \odot \bm{G}^{\mathrm{out}}_t)$\;
    $\bm{A}^{\mathrm{in}}_t \leftarrow - \bm{m}^{\mathrm{in}}_t / (\sqrt{\bm{v}^{\mathrm{in}}_t}+\epsilon)$, 
    $\bm{A}^{\mathrm{out}}_t \leftarrow - \bm{m}^{\mathrm{out}}_t / (\sqrt{\bm{v}^{\mathrm{out}}_t}+\epsilon)$\;
    \eIf{alternative update is used}{
        \eIf{$t$ is even}{
            $\alpha_t \leftarrow \frac{c\sqrt{d_{\mathrm{out}}d_{\mathrm{in}}}}{\|\bm{A}^{\mathrm{out}}_t\bm{W}_{t-1}\|_F+\epsilon}$\;
            $\bm{W}_t \leftarrow \mathcal{E}_2(\bm{A}^{\mathrm{out}}_t, \alpha_t)\bm{W}_{t-1}$ \;
        }{
            $\alpha_t \leftarrow \frac{ c\sqrt{d_{\mathrm{out}}d_{\mathrm{in}}}}{ \|\bm{W}_{t-1}\bm{A}^{\mathrm{in}}_t\|_F+\epsilon}$\;
            $\bm{W}_t \leftarrow\bm{W}_{t-1}\mathcal{E}_2(\bm{A}^{\mathrm{in}}_t, \alpha_t)$\;
        }
    }{
        $\alpha_t \leftarrow \frac{c\sqrt{d_{\mathrm{out}}d_{\mathrm{in}}}}{ \|\bm{A}^{\mathrm{out}}_t\bm{W}_{t-1}+\bm{W}_{t-1}\bm{A}^{\mathrm{in}}_t\|_F+\epsilon}$\;
        $\bm{W}_t \leftarrow \mathcal{E}_2(\bm{A}^{\mathrm{out}}_t, \alpha_t)\bm{W}_{t-1}\mathcal{E}_2(\bm{A}^{\mathrm{in}}_t, \alpha_t)$\;\vspace{-2mm}
    }}
\Return{$\bm{W}_t$}\;
\end{algorithm}
\end{minipage}
\vspace{-4mm}
\end{wrapfigure}

The previous exploration and experiments motivate four design choices.
First, scale-consistent rotational scaling is essential for stable training and large learning rates, whereas bilateral rotation balancing has a much weaker empirical effect. Second, Lie-algebra momentum better aligns with the geometry of spectrum-preserving updates than ambient-space accumulation. Third, alternate update retains most benefits of bilateral updates at substantially lower computational cost, though bilateral updates offer slightly better refinement near the final convergence. Empirically, we find that a second-order approximation of $\exp(\cdot)$ works sufficiently well for both optimization and spectrum preservation.
Based on these observations, the final Pion optimizer combines RMS-controlled scaling, Lie-algebra first-order and second-order momentum, and a second-order truncated approximation to the matrix exponential. Algorithm~\ref{alg:pion} gives the resulting optimizer steps.
Appendix~\ref{app:imp_details} presents an alternative implementation that uses transported ambient-space first-order momentum together with ambient-space second-order momentum.

\textbf{Computation Complexity}.
For a weight matrix $\bm{W}\in\mathbb{R}^{d_{\mathrm{out}}\times d_{\mathrm{in}}}$, the main overhead is constructing the input- and output-side Lie algebra gradients and applying the second-order exponential approximation. Computing
$
\bm{G}^{\mathrm{in}}=\bm{W}^\top\bm{G}-\bm{G}^\top\bm{W}
$
costs $4d_{\mathrm{out}}d_{\mathrm{in}}^2$ FLOPs, and computing
$
\bm{G}^{\mathrm{out}}=\bm{G}\bm{W}^\top-\bm{W}\bm{G}^\top
$
costs $4d_{\mathrm{out}}^2d_{\mathrm{in}}$ FLOPs. Element-wise momentum and second-moment updates are lower-order. RMS scaling requires evaluating $\bm{A}^{\mathrm{out}}\bm{W}+\bm{W}\bm{A}^{\mathrm{in}}$, costing another $2d_{\mathrm{out}}^2d_{\mathrm{in}}+2d_{\mathrm{out}}d_{\mathrm{in}}^2$ FLOPs. Applying the second-order update contributes $\mathcal{O}(d_{\mathrm{out}}^3+d_{\mathrm{in}}^3)$ FLOPs for squared Lie matrices and $2d_{\mathrm{out}}^2d_{\mathrm{in}}+2d_{\mathrm{out}}d_{\mathrm{in}}^2$ FLOPs for left and right multiplication with $\bm{W}$. Thus, the dominant additional cost of one bilateral Pion update is $\mathcal{O}(d_{\mathrm{out}}^2d_{\mathrm{in}}+d_{\mathrm{out}}d_{\mathrm{in}}^2+d_{\mathrm{out}}^3+d_{\mathrm{in}}^3)$. Alternate update can reduce the dominant update-side cost by roughly half. Compared with the baseline cost $\mathcal{O}(Bd_{\mathrm{out}}d_{\mathrm{in}})$ for the forward and backward passes of a linear layer with batch-token size $B$, the relative overhead of Pion is approximately
$\mathcal{O}(\frac{d_{\mathrm{out}}+d_{\mathrm{in}}}{B}+\frac{d_{\mathrm{out}}^3+d_{\mathrm{in}}^3}{Bd_{\mathrm{out}}d_{\mathrm{in}}})$.
In LLM pretraining, $B$ is typically large because it equals the number of tokens processed by the layer in one optimization step. Consequently, the optimizer-side matrix multiplications are amortized over a large token batch, and the overhead remains small relative to forward-backward computation. More detailed analysis and memory overhead are in Appendix~\ref{sec: computation overhead}.

\subsection{Compatibility with Maximal Update Parametrization}

Maximal update parametrization~\citep{yang2021tuning,yang2020feature,yang2023spectral} provides a principled framework for understanding training stability and enabling optimizer hyperparameter transfer across model scales.
A central insight of $\mu$P is that stable and effective training requires appropriate spectral scaling of both model weights and their updates. It identifies the following two spectral scaling conditions as essential for  training stability:

\vspace{-3.25mm}
\begin{equation}
\footnotesize
\!\text{Forward spectral condition:}~\underbrace{\left\|\bm{W}\right\|_2 = \Theta\!\left(\sqrt{{d_{\mathrm{out}}}/{d_{\mathrm{in}}}}\right)}_{\text{This is inherently satisfied by Pion.}}~~~\text{Update spectral condition:}~\underbrace{\left\|\Delta \bm{W}\right\|_2 = \Theta\!\left(\sqrt{{d_{\mathrm{out}}}/{d_{\mathrm{in}}}}\right)}_{\text{This is easily satisfied by Muon.}}.
\end{equation}
\vspace{-1.75mm}

Existing $\mu$P-compatible optimizers~\cite{liu2025muon,team2025kimi,xie2026controlled} are built on Muon, which inherently satisfies the update condition. As a result, prior work focuses on modifying Muon to also satisfy the forward condition. Pion takes the opposite route: it inherently satisfies the forward condition, so our goal is to make it satisfy the update condition. To this end, we approximate Pion’s weight-update magnitude as $\|\Delta \bm{W}_t\|_2 \approx \eta \|\bm{W}_t\|_2 (\|\bm{G}_t^{\mathrm{out}}\|_2 + \|\bm{G}_t^{\mathrm{in}}\|_2)$, where $\|\bm{W}_t\|_2$ naturally satisfies $\mu$P's spectral condition $\Theta(\sqrt{d_{\mathrm{out}}/d_{\mathrm{in}}})$. Thus, it suffices to maintain $\|\bm{G}_t^{\mathrm{out}}\|_2=\Theta(1)$ and $\|\bm{G}_t^{\mathrm{in}}\|_2=\Theta(1)$ for Pion to be $\mu$P-compatible. This can be achieved by normalizing the spectral norms of both factors, or alternatively by orthogonalizing $\bm{G}_t^{\mathrm{out}}$ and $\bm{G}_t^{\mathrm{in}}$. We give detailed derivation and full results in Appendix~\ref{app:mup}.

\begin{wrapfigure}{r}{0.37\linewidth}
\scriptsize
\centering
\vspace{0.1mm}
\includegraphics[width=1\linewidth]{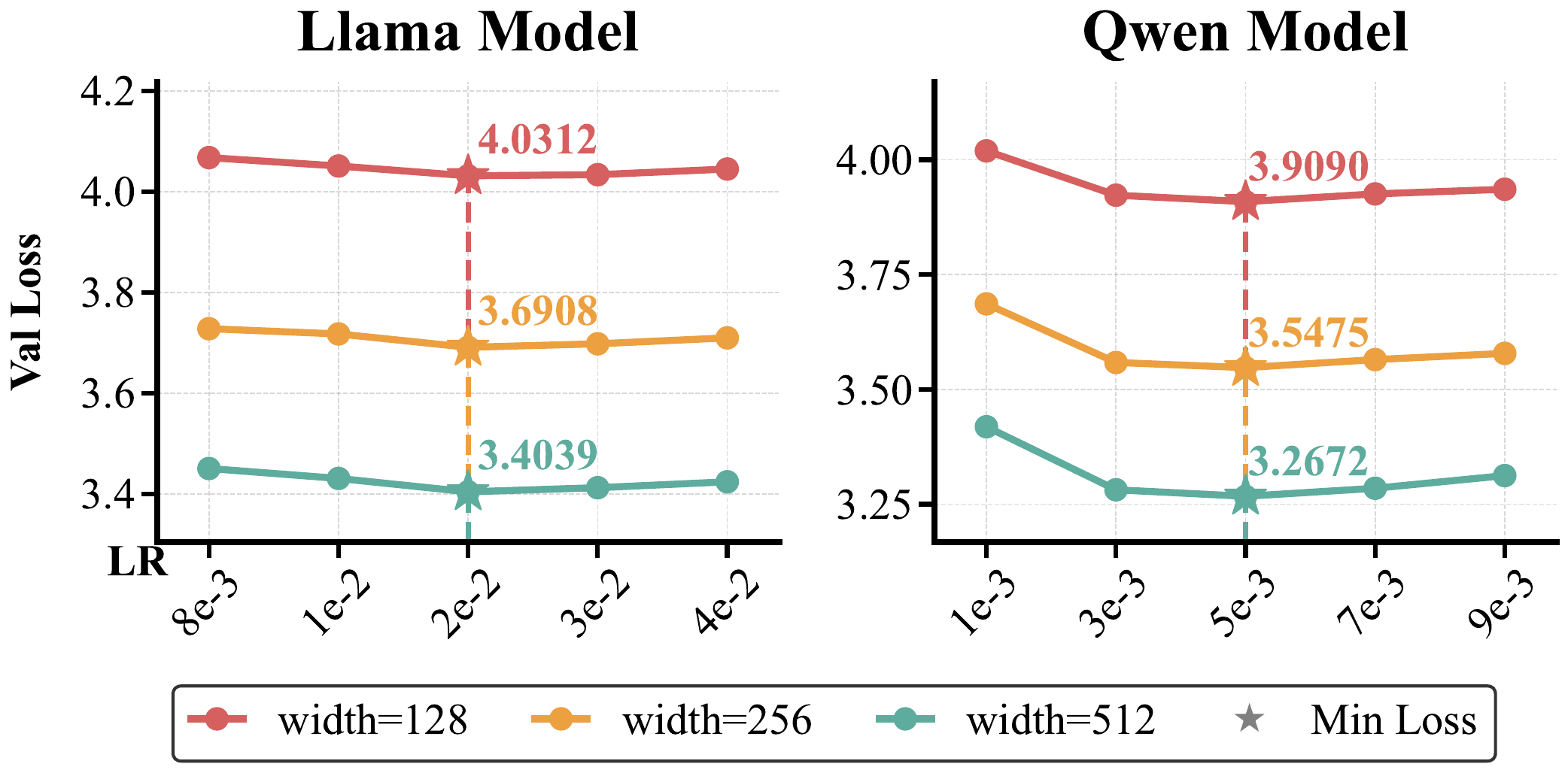}
\vspace{-4.35mm}
\caption{\scriptsize $\mu$P learning rate transfer across model scales.}
\label{fig:mup_main}
\vspace{1mm}
\end{wrapfigure}

To empirically assess the $\mu$P-compatibility of Pion, we conduct a hyperparameter transfer experiment across model scales. Specifically, we consider the Pion variant that applies simple normalization to both $\|\bm{G}_t^{\mathrm{out}}\|_2$ and $\|\bm{G}_t^{\mathrm{in}}\|_2$. We defer the detailed experimental settings and extended experimental results for the $\bm{G}_t$-orthogonalized variant to Appendix~\ref{app:mup}.
We train two representative LLM architectures (Llama and Qwen) over a range of model widths using the $\mu$P-compatible Pion, and examine whether optimal hyperparameters identified at small scale transfer reliably to larger models. Figure~\ref{fig:mup_main} validates that Pion has robust hyperparameter transferability, \ie, the optimal learning rate is invariant to model scales. In fact, a broader class of $\mu$P-compatible Pion variants remain unexplored, and our results represent only an initial step.

\begin{figure}[h!]
    \centering
    \setlength{\abovecaptionskip}{1pt}
    \setlength{\belowcaptionskip}{5pt}
    \vspace{3.25mm}
    \includegraphics[width=.99\linewidth]{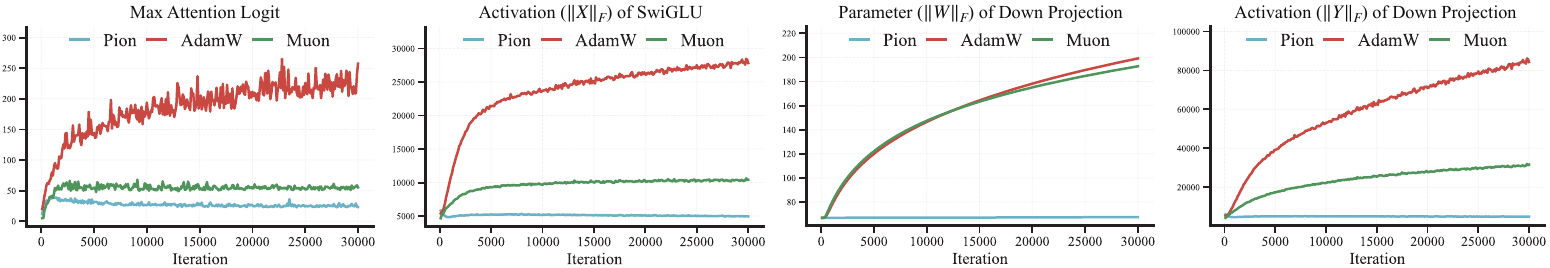}
    \caption{\scriptsize Dynamics of four diagnostic indicators for monitoring pretraining stability. From left to right, the four panels show, respectively, the maximum attention logit within the attention block of Layer 12, the norm of the input to the down-projection layer (equivalent to the SwiGLU activation output), the norm of the down-projection layer parameters, and the norm of its output. All indicators validate the strong stability of Pion.}
    \label{fig:stable_indicators}
\end{figure}

\section{Experiments and Results}

We evaluate Pion's performance and compare it with current standard optimizers (\eg, AdamW~\cite{loshchilov2019decoupled}, Muon~\cite{jordan2024muon,liu2025muon}) across diverse training scenarios, including both pretraining and post-training (supervised finetuning and reinforcement learning). All experimental details can be found in Appendix~\ref{app:exp}.

\subsection{Stable LLM Pretraining}

\setlength{\columnsep}{11pt}
\begin{wraptable}{r}[0cm]{0pt}
\scriptsize
\centering
\hspace{-2mm}
\setlength{\tabcolsep}{2pt}
\renewcommand{\arraystretch}{1.24}
    \begin{tabular}{l|ccccccccc|c}
    \specialrule{0em}{0pt}{-.5pt}
        \textbf{Method} & \textbf{ARC-C} & \textbf{ARC-E} & \textbf{BoolQ} & \textbf{Hella.} & \textbf{PIQA} & \textbf{SciQ} & \textbf{TriviaQA} & \textbf{Wino.} & \textbf{Avg} & \textbf{Val Loss} \\\shline
        AdamW & 25.94 & 45.96 & 46.30 & 45.10 & 71.27 & 70.80 & 1.06 & 51.46 & 44.74 & 2.7700 \\
        Muon & 25.34 & 47.94 & 51.56 & 46.70 & \textbf{72.20} & 71.60 & 1.64 & \textbf{53.75} & 46.34 & \textbf{2.7225} \\
        \rowcolor{gray!15} {Pion} & \textbf{26.79} & \textbf{49.41} & \textbf{57.58} & \textbf{47.34} & 71.27 & \textbf{73.40} & \textbf{2.17} & 53.59 & \textbf{47.69} & 2.7350 \\\specialrule{0em}{-5.25pt}{0pt}
    \end{tabular}
    \caption{\scriptsize Benchmark performance and validation loss on LLaMA-1.3B. The best results are in \textbf{Bold}.}
    \label{tab:norm_results}
    \vspace{2mm}
\end{wraptable}
\textbf{Main results on pretraining}. %
We first investigate the stability of Pion in LLM pretraining. Specifically, we pretrain a LLaMA-based 1.3B model (same as \citep{zhao2024galore,qiu2025reparameterized,gu2026mano}) on 54B tokens, 2$\times$ the Chinchilla-optimal data budget~\cite{hoffmann2022training}. We ensure that the number of training tokens are sufficient for the model to converge. These tokens are sampled from the C4 corpus~\citep{raffel2020exploring} and preprocessed using the T5-base tokenizer with sequence length 256. This setup corresponds to roughly 400K optimization iterations, effectively simulating practice long-horizon training regimes. We compare Pion (with alternate update and Lie-Lie momentum) against two most widely used optimizers, AdamW and Muon, under the same training configuration. Main results are given in Table~\ref{tab:norm_results}. Pion achieves the best generalization performance across the evaluated benchmarks. Its validation loss is comparable to that of Muon, and both Pion and Muon yield substantially lower validation losses than AdamW. These results demonstrate that Pion not only scales effectively to large language model pretraining but also provides a superior trade-off between optimization stability and downstream capabilities.

\begin{wrapfigure}{r}{0.3\linewidth}
\scriptsize
\centering
\vspace{-0.75mm}
\includegraphics[width=1.0\linewidth]{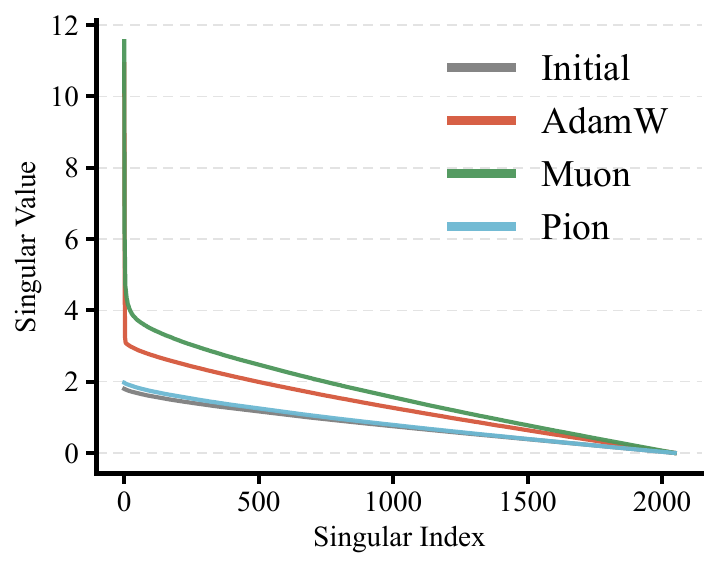}
\vspace{-6mm}
\caption{\scriptsize Weight spectrum comparison.}
\label{fig:spectrum_o}
\end{wrapfigure}
Besides validation loss, we also monitor several indicators of training stability  in Figure~\ref{fig:stable_indicators}. 
Specifically, we track the maximum attention logit~\citep{liu2025muon}, the norm of the SwiGLU activation $\|\bm{X}\|_F$ (also the input of the down-projection layer), the norm of down-projection weights $\|\bm{W}\|_F$, 
and the norm of down-projection outputs $\|\bm{Y}\|_F$, following \citep{wortsman2023small,dehghani2023scaling,rybakov2024methods}. 
In particular, the norm of the SwiGLU activation and the maximum attention logit have been generally recognized as an important stability indicator in large-scale pretraining~\citep{gemmateam2024gemma2improvingopen,deepseekai2025deepseekv3technicalreport,openai2025gptoss120bgptoss20bmodel,deepseekai2026deepseekv4}.
These results reveal a clear separation among optimizers. AdamW exhibits continuously growing attention logits and rapidly amplified activation magnitude. Muon substantially suppresses attention-logit growth, but its activations and down-projection norms keep increasing throughout training. In contrast, Pion keeps all monitored quantities nearly flat and stable. 
This distinctive behavior highlights Pion's potential for stable large-scale training.
Pion's training stability stems from its well-preserved weight matrix spectra throughout the optimization process, as verified in Figure~\ref{fig:spectrum_o}. Full results are provided in Appendix~\ref{sec:more_results_in_pretraining}.

\begin{wrapfigure}{r}{0.3\linewidth}
\scriptsize
\centering
\vspace{-1mm}
\includegraphics[width=1.0\linewidth]{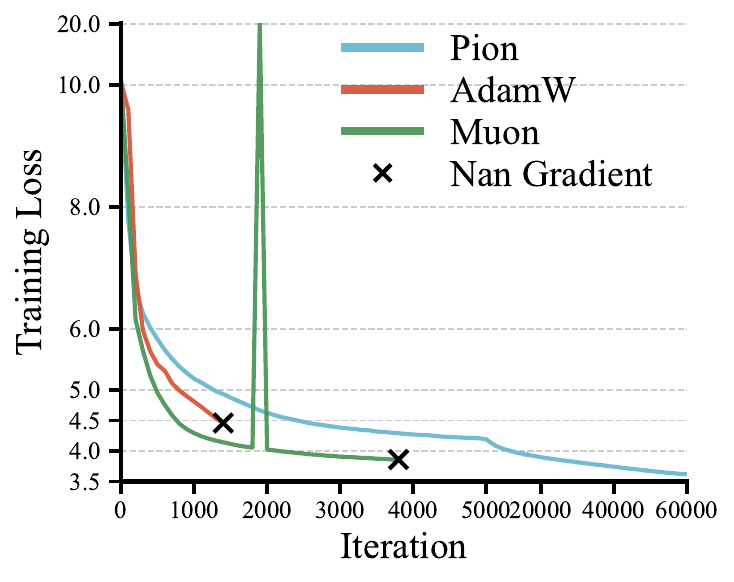}
\vspace{-5.75mm}
\caption{\scriptsize Normalization-free pretraining.}
\label{fig:nonorm}
\end{wrapfigure}

\vspace{-0.25mm}
\textbf{Normalization-free pretraining}.
To stress-test Pion's training stability, we remove all normalization layers from the 60M LLaMA-based model. Because normalization layers~\citep{ba2016layer,zhang2019root} are widely regarded as essential for controlling activation scales and stabilizing gradient back-propagation, this challenging setting can effectively probe whether an optimizer alone can provide adequate scale regulation during pretraining.
Under this setting, both AdamW and Muon can make initial progress but soon fail due to gradient overflow, producing NaNs. In contrast, Pion remains stable throughout the full 9.6B-token training run and converges successfully. These results show that spectrum-preserving updates can limit signal amplification even when explicit normalization mechanisms are removed. They further suggest that spectrum-preserving optimization can partially replace architectural scale-control mechanisms, providing an optimizer-level source of training stability.

\begin{wrapfigure}{r}{0.3\linewidth}
\scriptsize
\centering
\vspace{-1mm}
\includegraphics[width=\linewidth]{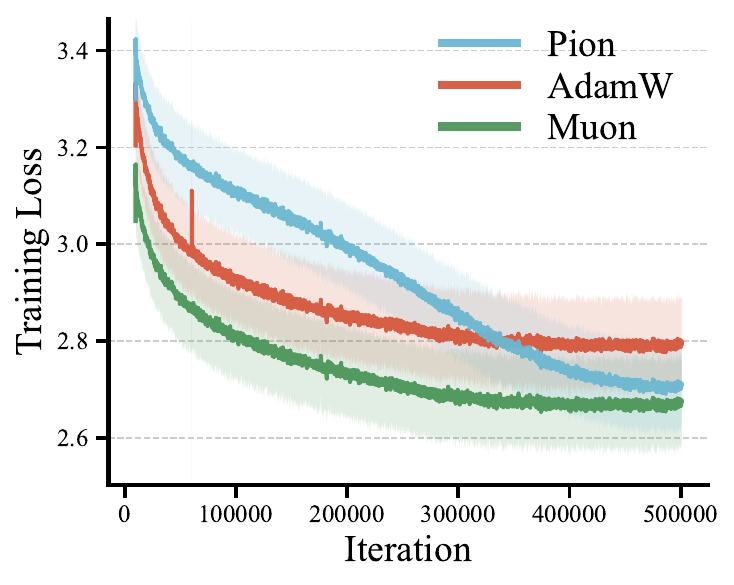}
\vspace{-6mm}
\caption{\scriptsize Training Loss of DeepNet.}
\label{fig:deepnet}
\end{wrapfigure}
\vspace{2mm}
\vspace{-0.5em}

\vspace{-0.25mm}
\textbf{Pretraining on ultra-deep architectures}.
We further stress-test stability under LLMs with extreme depth. Training such networks often leads to severe optimization instabilities, including vanishing gradients and representation collapse~\citep{sun2025curse}.
To examine this setting, we scale the depth of a LLaMA 60M baseline from 8 to 200 layers and train each model on a 50B-token subset of the C4 dataset. For visual clarity, Fig.~\ref{fig:deepnet} reports the training loss after applying an $N$-step moving average, $Y_t = \frac{1}{N} \sum_{i=0}^{N-1} X_{t-i}$, where $X_t$ and $Y_t$ denote the raw and smoothed losses, respectively. 
We quantify stability by the mean standard deviation of the local loss trajectory, visualized as the integrated area of the shaded band.
AdamW shows the largest loss spikes and the worst overall stability. The mean standard deviations are 0.0931 for AdamW, 0.0927 for Muon, and 0.0892 for Pion, making Pion the most stable optimizer under this setting. Pion also decreases loss more rapidly in the intermediate training stages, exhibiting an effective optimization behavior even under extreme depth.

\begin{wrapfigure}{r}{0.3\linewidth}
\scriptsize
\centering
\vspace{-1mm}
\includegraphics[width=1.0\linewidth]{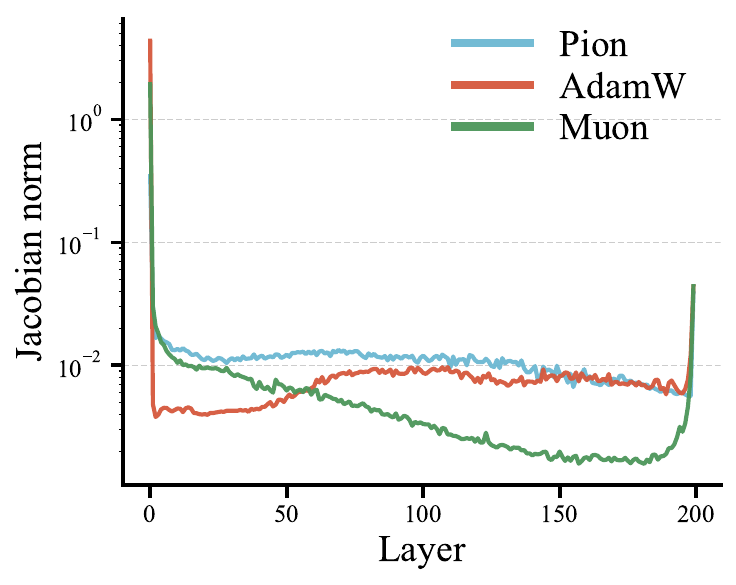}
\vspace{-6mm}
\caption{\scriptsize Jacobian Norm in DeepNet.}
\label{fig:jacobian_norm}
\vspace{2mm}
\end{wrapfigure}

\vspace{-0.5mm}
To further justify our results, we analyze the layer-wise expressivity induced by different optimizers. Following the anlysis in~\citep{muhtar2026does,sun2025curse}, we quantify each layer's local geometry using a shape-normalized Frobenius distance $\Vert \bm{J}_{\ell}-\bm{I}\Vert_{F}$, where $\bm{J}_{\ell}$ is the Jacobian matrix of layer $\ell$. A larger Jacobian norm signifies a greater deviation from identity-like transport, thereby reflecting more effective expressivity. As shown in Figure~\ref{fig:jacobian_norm}, all optimizers exhibit a substantial deviation at the first layer, reflecting the architectural necessity of transforming static embeddings into highly contextualized representations. However, as the network enters the deep iterative refinement phase, the expressivity induced by different optimizers diverges significantly. 
AdamW exhibits a sharp Jacobian-norm drop in the middle layers, while Muon’s norm steadily decays with depth. In contrast, Pion leverages the full network depth, maintaining a notably uniform norm profile across layers and consistently dominating the interior. These results suggest that Pion preserves layer-wise balanced expressivity in deep networks, avoiding the expressivity degradation observed in AdamW and Muon.

\subsection{Efficient LLM Post-training}

\setlength{\columnsep}{11pt}
\begin{wraptable}{r}[0cm]{0pt}
\scriptsize
\centering
\hspace{-2mm}
\setlength{\tabcolsep}{2pt}
\renewcommand{\arraystretch}{1.24}
    \begin{tabular}{l|cccc|cccc}
    \specialrule{0em}{0pt}{-2pt}
        \multirow{3}{*}{\textbf{Method}} & \multicolumn{4}{c|}{\textbf{Qwen2.5-1.5B}} & \multicolumn{4}{c}{\textbf{Llama3.2-3B}} \\
        & \multicolumn{2}{c}{Math} & \multicolumn{2}{c|}{Code} & \multicolumn{2}{c}{Math} & \multicolumn{2}{c}{Code} \\
        & ID (\%) & OOD (\%) & ID (\%) & OOD (\%) & ID (\%) & OOD (\%) & ID (\%) & OOD (\%) \\\shline
        Base    & 59.81 & 64.83 & 35.98 & 63.99 & 25.47 & 67.59 & 26.22 & 53.08 \\
        AdamW & \textbf{65.88} & 62.13 & 51.83 & 62.64 & \textbf{59.87} & 60.86 & 46.95 & 58.64 \\
        Muon & 65.27 & 61.22 & 50.00 & 62.41 & 57.77 & \textbf{61.20} & 46.34 & 58.88 \\
        \rowcolor{gray!15} {Pion} & 65.76 & \textbf{62.16} & \textbf{53.05} & \textbf{63.21} & 58.83 & 60.44 & \textbf{47.19} & \textbf{59.74} \\\specialrule{0em}{-5.25pt}{0pt}
    \end{tabular}
    \caption{\scriptsize Performance comparison of Pion and baseline optimizers on finetuning tasks. ID (in-domain) evaluates downstream capability, while OOD (out-of-domain) measures the model’s robustness against forgetting. \textbf{Bold} values indicate the best results.}
    \label{tab:finetuning}
    \vspace{2mm}
\end{wraptable}

\textbf{Supervised finetuning}. In this section, we study the effectiveness of the Pion optimizer in the LLM supervised finetuning scenarios. Specifically, we conduct full-parameter finetuning experiments utilizing the LLaMA-Factory~\citep{zheng2024llamafactory} framework. We employ Qwen2.5-1.5B~\citep{team2024qwen2} and Llama-3.2-3B~\citep{grattafiori2024llama} as base models, fine-tuning them on the  MetaMathQA~\citep{yu2023metamath} and Magicoder-Evol-Instruct-110K~\citep{wei2024magicoder} datasets. For our evaluation protocol, following the setups in~\citep{biderman2024lora}, we analyze the stability-plasticity tradeoff~\citep{mermillod2013stability} inherent in model adaptation. We adopt GSM8K~\citep{cobbe2021training} and HumanEval~\citep{chen2021evaluating} as in-domain (ID) benchmarks for mathematical and code generation tasks respectively. Concurrently, out-of-domain (OOD) performance is assessed using ARC-Easy, ARC-Challenge~\citep{allenai:arc}, Winogrande~\citep{ai2:winogrande}, PIQA~\citep{Bisk2020} and Hellaswag~\citep{zellers2019hellaswag} benchmarks. All evaluations are conducted using the LM Evaluation Harness~\citep{eval-harness} framework with its default generation parameters. More experiment details can be found in Appendix~\ref{app:exp sft}.
The experimental results are summarized in Table \ref{tab:finetuning}. Overall, Pion achieves a highly competitive stability-plasticity tradeoff compared to established baselines such as AdamW and Muon. In particular, it shows a clear performance advantage on code generation, achieving the highest ID and OOD scores across both base models. On mathematical finetuning, Pion matches the ID performance of competing optimizers while more effectively preserving OOD capabilities, highlighting its robustness against catastrophic forgetting. This also validates the effectiveness of orthogonal finetuning~\cite{qiu2023controlling,liu2024boft,qiu2025oftv2} in knowledge preservation.

\begin{table*}[t!]
    \centering
    \scriptsize
    \renewcommand{\arraystretch}{1.2}
    \setlength{\tabcolsep}{2.9pt}
    \begin{tabular}{l|cccccc|cccccc}
    \specialrule{0em}{0pt}{-4pt}
        \multirow{4}{*}{\textbf{Method}} & \multicolumn{6}{c|}{\textbf{Qwen3-1.7B}} & \multicolumn{6}{c}{\textbf{DeepSeek-R1-Distill-Qwen-1.5B}} \\
        & \makecell{AIME24 \\ (avg@32)} & \makecell{AIME25 \\ (avg@32)} & \makecell{AMC23 \\ (avg@8)} & \makecell{Minerva \\ Math \\ (avg@4)} & \makecell{Olympiad \\ Bench \\ (avg@8)} & Avg 
        & \makecell{AIME24 \\ (avg@32)} & \makecell{AIME25 \\ (avg@32)} & \makecell{AMC23 \\ (avg@8)} & \makecell{Minerva \\ Math \\ (avg@4)} & \makecell{Olympiad \\ Bench \\ (avg@8)} & Avg \\\shline
        Base    & 4.06 & 10.10 & 30.27 & 16.27 & 23.67 & 16.87 & 20.52 & 20.83 & 54.06 & 19.39 & 36.20 & 30.20 \\
        AdamW & 22.71 & 20.94 & 58.43 & 25.91 & 46.09 & 34.82 & 25.42 & 23.94 & 62.65 & 23.16 & 44.69 & 35.97 \\
        Muon & 20.42 & 19.27 & 54.22 & 24.08 & 42.41 & 32.08 & 29.06 & 23.33 & 66.72 & 22.89 & 44.61 & 37.32 \\
        \rowcolor{gray!15} {Pion} & \textbf{25.42} & \textbf{21.98} & \textbf{59.94} & \textbf{26.84} & \textbf{46.43} & \textbf{36.12} & \textbf{30.00} & \textbf{24.38} & \textbf{66.87} & \textbf{23.90} & \textbf{46.43} & \textbf{38.32} \\ \specialrule{0em}{-5.25pt}{0pt}
    \end{tabular}
    \caption{\scriptsize Performance comparison of Pion and other baseline optimizers on RLVR tasks (training with GRPO). The metric avg@K denotes the average accuracy across K generated samples per problem. \textbf{Bold} values indicate the best overall average results.}
    \label{tab:rlvr}
    \vspace{-3.5mm}
\end{table*}

\textbf{Reinforcement learning with verifiable reward (RLVR)}.
We further investigate Pion as an optimizer for RLVR. Our motivation comes from recent observations~\citep{zhu2025path,huang2026peft} that RLVR updates largely preserve the spectral structure of pretrained weight matrices.
This finding suggests that RLVR may particularly benefit from optimizers whose update geometry is aligned with the intrinsic spectral structure of the pretrained weights.
Since Pion explicitly preserves the weight spectrum throughout optimization, its spectrum-preserving update mechanism provides a natural inductive bias for RLVR.
Motivated by this connection, we compare Pion with AdamW and Muon to evaluate its effectiveness in RLVR training.

\begin{wrapfigure}{r}{0.4\linewidth}
\scriptsize
\centering
\vspace{-1mm}
\includegraphics[width=.99\linewidth]{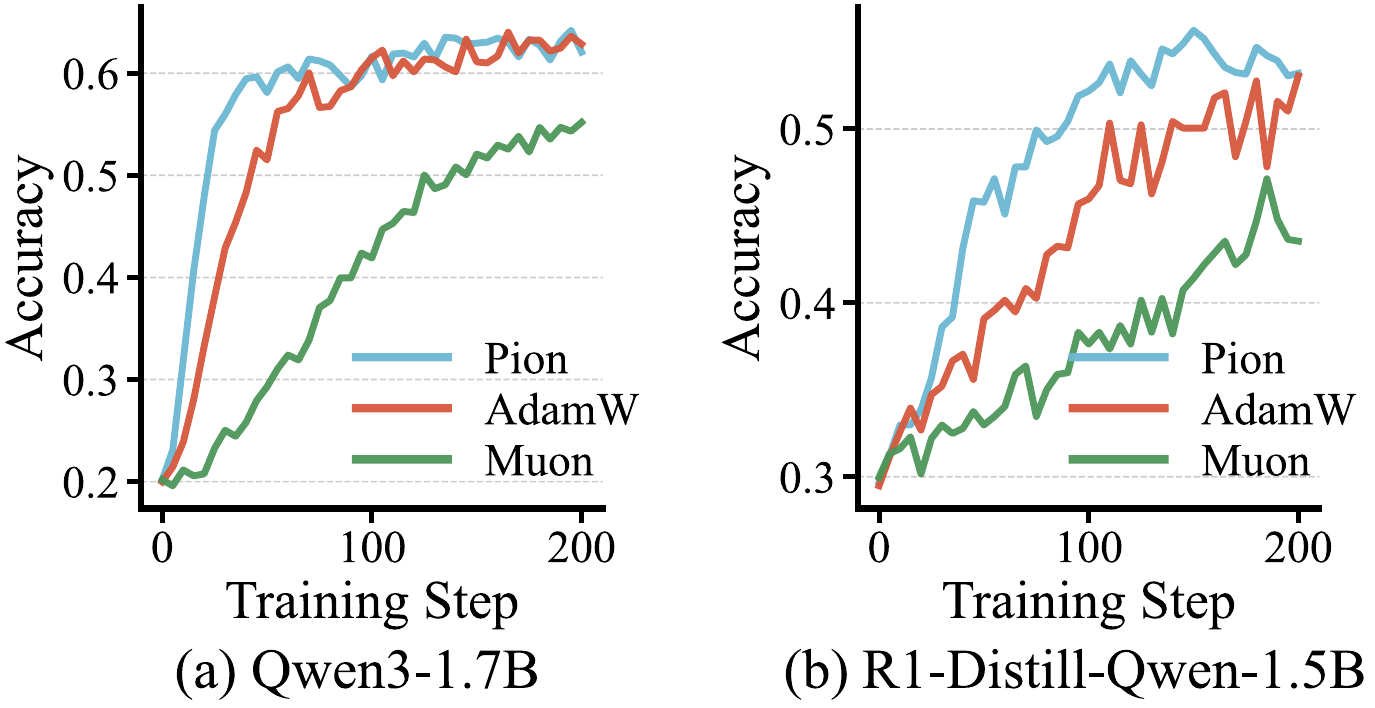}
\vspace{-1.75mm}
\caption{\scriptsize Training dynamics of evaluation accuracy. Pion demonstrates the fastest convergence rate among all.}
\label{fig:rl acc}
\vspace{1.5mm}
\end{wrapfigure}

We implement the RLVR training pipeline with the VeRL framework~\citep{sheng2024hybridflow} and adopt Group Relative Policy Optimization (GRPO)~\citep{shao2024deepseekmath} as our RL objective. Experiments are conducted on two base models, Qwen3-1.7B~\citep{yang2025qwen3} and DeepSeek-R1-Distill-Qwen-1.5B~\citep{guo2025deepseek}, using DeepMath~\citep{he2025deepmath} as the training dataset. We evaluate the trained models on five mathematical reasoning benchmarks, including AIME24~\citep{AIME24}, AIME25~\citep{AIME25}, AMC23~\citep{AMC23}, Minerva Math~\citep{lewkowycz2022solving}, and OlympiadBench~\citep{he2024olympiadbench}. The maximum generation length is set to 4096 tokens for Qwen3-1.7B and 8192 tokens for DeepSeek-R1-Distill-Qwen-1.5B. All evaluations are performed using POLARIS~\citep{Polaris2025}. Detailed settings and more results are provided in Appendix~\ref{app:exp rl}.
As shown in Table~\ref{tab:rlvr}, Pion achieves the best performance in both settings. 
These results provide empirical support for the hypothesis that spectrum-preserving optimization is particularly well suited to RLVR, where parameter updates tend to preserve the spectral structure of pretrained weights.
Notably, Pion's improvements are consistent across multiple evaluation metrics and checkpoints, suggesting that its spectrum-preserving constraint provides a systematic advantage rather than an isolated performance gain. 
Furthermore, Figure~\ref{fig:rl acc} demonstrates that Pion can achieve significantly faster convergence in validation accuracy.

\section{Balanced Rotations in Pion via Matrix Sign Normalization}
\label{sec:further-exploration}
\begin{wrapfigure}{r}{0.3\linewidth}
\scriptsize
\centering
\vspace{-.75mm}
\includegraphics[width=.99\linewidth]{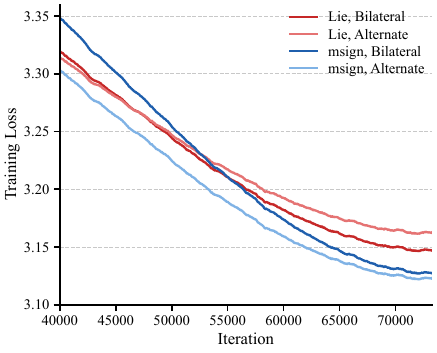}
\vspace{-5.25mm}
\caption{\scriptsize Performance of Pion with Lie algebra and $\operatorname{msign}(\cdot)$ on Fine-web Edu~\cite{penedo2024fineweb}.}
\label{fig:msign dynamics}
\vspace{1mm}
\end{wrapfigure}

We explore applying the matrix sign normalization (msign normalization, $\operatorname{msign}(\cdot)$) to Pion while keeping its spectrum-preserving construction. We apply $\operatorname{msign}$ to the skew-symmetric update generators on both sides of Pion. This keeps the generator skew-symmetric while giving it approximately unit singular values, so every active rotation plane turns by approximately the same angle $\eta$, mirroring Muon's equal-singular-value updates. The update no longer depends on the gradient scale, so the step size is set by $\eta$ alone, and the RMS scaling and second-moment statistics are not needed. The variant keeps one ambient momentum buffer per matrix. We visualize the training dynamics of Pion optimizer (msign) in Fig.~\ref{fig:msign dynamics}. Msign normalization is also compatible with Pion and achieves better convergence speed.  This shows a promising direction for more advanced spectrum-preserving optimizer. A detailed algorithm in Appendix~\ref{app:msign}.

\section{Related Work and Concluding Remarks}

\textbf{Related work}. 
Stable LLM training relies on a delicate interplay between optimization dynamics and scale control. A common recipe combines adaptive optimizers such as AdamW~\cite{kingma2014adam} with normalization layers~\cite{ba2016layer,henry2020query} or scale-aware parameterizations such as maximal update parameterization~\cite{yang2020feature,yang2023spectral}. Recent work has further shown that matrix-aware optimization can substantially improve training stability by exploiting the structure of weight matrices~\cite{vyas2024soap,gupta2018shampoo,tuddenham2022orthogonalising,pethick2025training,jordan2024muon,carlson2015stochastic,carlson2015stochastic2,bernstein2024old,bernstein2024modular}. In particular, Muon~\cite{jordan2024muon} and its variants~\cite{liu2025muon,li2025normuon,team2025kimi,xie2026controlled,amsel2025polar,chen2025muon} have attracted considerable attention for their empirical effectiveness.
Our work shares the goal of stable training, but approaches it from a distinct geometric perspective. Building on ideas from orthogonal group optimization~\cite{lezcano2019cheap,mhammedi2017efficient} and the geometric inductive biases of POET~\cite{qiu2026poetx,qiu2025reparameterized}, Pion elevates spectrum preservation to an optimizer-level principle. Rather than enforcing orthogonality through reparameterization, Pion directly preserves the spectral geometry of weight matrices during optimization. This view also offers a natural path to minimum hyperspherical energy~\cite{liu2018mhe,Liu2021OPT,Liu2021SphereUni}, while retaining sufficient flexibility for effective training.

\textbf{Concluding remarks}. 
We introduce Pion, a spectrum-preserving optimizer for stable training via orthogonal equivalence transformations. Pion provably preserves the weight spectrum and maintains minimal hyperspherical energy throughout training without relying on explicit reparameterization. Empirical results have demonstrated Pion's effectiveness, making it competitive with standard optimizers across diverse settings, including pretraining, supervised finetuning, and reinforcement learning.

\section*{Acknowledgments}
The work described in this paper was supported by the General Research Fund and Early Career Scheme by the Research Grants Council of Hong Kong (Project Number: 24211626). The authors declare no competing interests.

\bibliographystyle{plain}
\bibliography{reference}

\clearpage

\newpage
\appendix
\onecolumn
\addcontentsline{toc}{section}{Appendix}
\renewcommand \thepart{}
\renewcommand \partname{}
\part{\Large{\centerline{Appendix}}}
\parttoc

\newpage

\newpage
\section{Geometric Structure of Pion's Update}
\subsection{Derivation of the Update Rule}
\label{app:eq1-derivation}

Let the objective be $f(\bm{W})$ and denote $\bm{G}_t=\nabla f(\bm{W}_t)$. Starting from $\bm{W}_t=\bm{I}_{d_{\mathrm{out}}}\bm{W}_t\bm{I}_{d_{\mathrm{in}}}$, we temporarily regard the two identity matrices as independent variables $\bm{Q}_{\mathrm{out}}$ and $\bm{Q}_{\mathrm{in}}$, and define $\widetilde{\bm{W}}(\bm{Q}_{\mathrm{out}},\bm{Q}_{\mathrm{in}})=\bm{Q}_{\mathrm{out}}\bm{W}_t\bm{Q}_{\mathrm{in}}$. The gradients are evaluated at $\bm{Q}_{\mathrm{out}}=\bm{I}$ and $\bm{Q}_{\mathrm{in}}=\bm{I}$. For infinitesimal perturbations $\Delta\bm{Q}_{\mathrm{out}}$ and $\Delta\bm{Q}_{\mathrm{in}}$, the corresponding perturbation of the weight matrix is $\Delta\bm{W}=\Delta\bm{Q}_{\mathrm{out}}\bm{W}_t+\bm{W}_t\Delta\bm{Q}_{\mathrm{in}}$. By the chain rule and the Frobenius inner product,
\begin{align*}
df
&=\langle\bm{G}_t,\Delta\bm{W}\rangle_F
=\langle\bm{G}_t,\Delta\bm{Q}_{\mathrm{out}}\bm{W}_t\rangle_F+\langle\bm{G}_t,\bm{W}_t\Delta\bm{Q}_{\mathrm{in}}\rangle_F \\
&=\langle\bm{G}_t\bm{W}_t^\top,\Delta\bm{Q}_{\mathrm{out}}\rangle_F+\langle\bm{W}_t^\top\bm{G}_t,\Delta\bm{Q}_{\mathrm{in}}\rangle_F .
\end{align*}
Therefore, the Euclidean gradients with respect to the two identity factors are
\begin{equation}
\nabla_{\bm{Q}_{\mathrm{out}}}f=\bm{G}_t\bm{W}_t^\top,
\qquad
\nabla_{\bm{Q}_{\mathrm{in}}}f=\bm{W}_t^\top\bm{G}_t .
\end{equation}

Next, we impose the orthogonality constraint. Since $\bm{Q}_{\mathrm{out}}\in\mathrm{O}(d_{\mathrm{out}})$ and $\bm{Q}_{\mathrm{in}}\in\mathrm{O}(d_{\mathrm{in}})$, the Euclidean gradients must be projected onto the tangent spaces at the identity. The tangent space of $\mathrm{O}(k)$ at the identity is the Lie algebra $\mathfrak{so}(k)=\{\bm{\Omega}\in\mathbb{R}^{k\times k}:\bm{\Omega}^\top=-\bm{\Omega}\}$ (cf.\ Lemma~\ref{lemma:tangent-space-isospectral}). Under the Frobenius inner product, the orthogonal projection of a square matrix $\bm{A}$ onto $\mathfrak{so}(k)$ is $\operatorname{skew}(\bm{A})=\tfrac12(\bm{A}-\bm{A}^\top)$. This gives the Lie-algebra directions $\operatorname{skew}(\bm{G}_t\bm{W}_t^\top)$ and $\operatorname{skew}(\bm{W}_t^\top\bm{G}_t)$. Eq.~\eqref{eq:update} uses the unnormalized forms
\begin{equation}
\bm{G}_t^{\mathrm{out}}=\bm{G}_t\bm{W}_t^\top-(\bm{G}_t\bm{W}_t^\top)^\top=2\operatorname{skew}(\bm{G}_t\bm{W}_t^\top),
\qquad
\bm{G}_t^{\mathrm{in}}=\bm{W}_t^\top\bm{G}_t-(\bm{W}_t^\top\bm{G}_t)^\top=2\operatorname{skew}(\bm{W}_t^\top\bm{G}_t),
\end{equation}
which differ from the standard projections only by the constant factor $2$, absorbed into the learning rate $\eta$.

Finally, $\bm{G}_t^{\mathrm{out}}\in\mathfrak{so}(d_{\mathrm{out}})$ and $\bm{G}_t^{\mathrm{in}}\in\mathfrak{so}(d_{\mathrm{in}})$, so $\exp(-\eta\bm{G}_t^{\mathrm{out}})\in\mathrm{O}(d_{\mathrm{out}})$ and $\exp(-\eta\bm{G}_t^{\mathrm{in}})\in\mathrm{O}(d_{\mathrm{in}})$. Applying these as the new identity factors gives
\begin{equation}
\bm{W}_{t+1}=\exp(-\eta\bm{G}_t^{\mathrm{out}})\,\bm{W}_t\,\exp(-\eta\bm{G}_t^{\mathrm{in}}),
\end{equation}
which is Eq.~\eqref{eq:update}. Since the identity factors are reset to $\bm{I}$ at every iteration, each step is a Riemannian gradient step on $\mathrm{O}(d_{\mathrm{out}})\times\mathrm{O}(d_{\mathrm{in}})$, with the exponential map as the retraction.

\subsection{Singular-Value Preservation and Planar Rotations}
\label{app:pion_geometry_derivation}
We provide the detailed derivation of Proposition~\ref{prop:total_rotation}. Recall that Pion updates a matrix parameter as
\begin{equation}
    \bm{W}_{t+1}
    =
    \bm{R}_t\bm{W}_t\bm{P}_t,
    \qquad
    \bm{R}_t=\exp(-\eta\bm{G}^{\mathrm{out}}_t),
    \quad
    \bm{P}_t=\exp(-\eta\bm{G}^{\mathrm{in}}_t).
\end{equation}
Since $\bm{G}^{\mathrm{out}}_t$ and $\bm{G}^{\mathrm{in}}_t$ are skew-symmetric, $\bm{R}_t$ and $\bm{P}_t$ are orthogonal. Therefore, if $\bm{W}_t=\bm{U}_t\bm{\Sigma}_0\bm{V}_t^\top$ is a singular value decomposition, then
\begin{equation}
    \bm{W}_{t+1}
    =
    (\bm{R}_t\bm{U}_t)
    \bm{\Sigma}_0
    (\bm{P}_t^\top\bm{V}_t)^\top .
\end{equation}
This is again a valid singular value decomposition up to orthogonal basis changes, so the singular values $\bm{\Sigma}_0$ are preserved. The update hence only rotates the left and right singular subspaces of $\bm{W}_t$.

This also explains why $\|\Delta\bm{W}\|_F$, with $\Delta\bm{W}=\bm{W}_{t+1}-\bm{W}_t$, measures the total strength of the update. Orthogonal multiplication preserves vector norms, so the update does not scale the row or column vectors of $\bm{W}_t$; it changes only their directions. Thus, the Frobenius norm of the displacement reflects the aggregate angular movement induced by the two rotations.

We next make the planar-rotation structure explicit. For simplicity, consider the one-sided update
\begin{equation}
    \bm{W}_{t+1}
    =
    \bm{W}_t\exp(-\eta\bm{G}^{\mathrm{in}}_t).
\end{equation}
Because $\bm{G}^{\mathrm{in}}_t$ is real skew-symmetric, there exists an orthogonal matrix $\bm{Q}\in O(d_{\mathrm{in}})$ such that
\begin{equation}
    \bm{G}^{\mathrm{in}}_t
    =
    \bm{Q}
    \mathrm{diag}\left(
    \begin{pmatrix}0&\theta_1\\-\theta_1&0\end{pmatrix},
    \cdots,
    \begin{pmatrix}0&\theta_m\\-\theta_m&0\end{pmatrix},
    \cdots
    \right)
    \bm{Q}^\top
    =
    \bm{Q}\bm{\Sigma}\bm{Q}^\top .
\end{equation}
Using $\bm{W}_t=\bm{U}_t\bm{\Sigma}_0\bm{V}_t^\top$, we obtain
\begin{equation}
    \bm{W}_t\exp(-\eta\bm{G}^{\mathrm{in}}_t)
    =
    \bm{U}_t\bm{\Sigma}_0\bm{V}_t^\top
    \bm{Q}
    \exp(-\eta\bm{\Sigma})
    \bm{Q}^\top .
\end{equation}
The matrix $\exp(-\eta\bm{\Sigma})$ is block diagonal, and each $2\times2$ block has the form
\begin{equation}
    \exp\left(
    -\eta
    \begin{pmatrix}0&\theta_j\\-\theta_j&0\end{pmatrix}
    \right),
\end{equation}
which is a planar rotation with angle determined by $-\eta\theta_j$. Hence the right-side update first represents the right singular space in the orthogonal basis induced by $\bm{G}^{\mathrm{in}}_t$, and then applies independent planar rotations within the corresponding $2$D invariant subspaces.

The Frobenius and spectral norms of the Lie algebra element quantify these rotations:
\begin{equation}
    \|-\eta\bm{G}^{\mathrm{in}}_t\|_F
    =
    \eta\sqrt{2\sum_{j=1}^{\lfloor d_{\mathrm{in}}/2\rfloor}\theta_j^2},
    \qquad
    \|-\eta\bm{G}^{\mathrm{in}}_t\|_2
    =
    \eta\max_j |\theta_j|.
\end{equation}
Thus, $\|\bm{G}^{\mathrm{in}}_t\|_F/\sqrt{d_{\mathrm{in}}}$ characterizes the average rotation magnitude on the input side up to a constant factor, while $\|\bm{G}^{\mathrm{in}}_t\|_2$ controls the maximum angle. The same argument applies to $\bm{G}^{\mathrm{out}}_t$ for the output-side update.

\newpage
\section{Convergence Analysis}
\label{app:convergence}
Before presenting the convergence analysis, we first introduce several basic characterizations of the geometry induced by spectrum-preserving updates in Equation \ref{eq:update}. As this update preserves the spectrum of the weight matrices, idealized trajectory stays on the set of matrices sharing the same singular values as the initialization. We denote this set by $\mathcal{M}_{\bm{W}_0}$.

\begin{lemma}[Characterization of the Isospectral Manifold]
Let $\bm{W}_0 \in \mathbb{R}^{m \times n}$ be the initial weight matrix. The isospectral manifold passing through $\bm{W}_0$ is
\begin{equation}
    \mathcal{M}_{\bm{W}_0}
    =
    \left\{
    \bm{U}\bm{W}_0\bm{V}^\top
    \mid
    \bm{U}\in \mathrm{O}(m),\,
    \bm{V}\in \mathrm{O}(n)
    \right\}.
\end{equation}
\label{lemma:character-isospectral-manifold}
\end{lemma}

\begin{proof}
Let $\bm{W}_0=\bm{U}_0\bm{\Sigma}\bm{V}_0^\top$ be the singular value decomposition of $\bm{W}_0$.

First, if $\bm{W}\in \mathcal{M}_{\bm{W}_0}$, then $\bm{W}$ has the same singular values as $\bm{W}_0$. Hence we can write $\bm{W}=\bm{U}_w\bm{\Sigma}\bm{V}_w^\top$. Substituting $\bm{\Sigma}=\bm{U}_0^\top\bm{W}_0\bm{V}_0$ gives
\begin{equation}
    \bm{W}
    =
    \bm{U}_w\bm{U}_0^\top
    \bm{W}_0
    \bm{V}_0\bm{V}_w^\top
    =
    \bm{U}\bm{W}_0\bm{V}^\top,
\end{equation}
where $\bm{U}=\bm{U}_w\bm{U}_0^\top\in \mathrm{O}(m)$ and
$\bm{V}=\bm{V}_w\bm{V}_0^\top\in \mathrm{O}(n)$.

Conversely, if $\bm{W}=\bm{U}\bm{W}_0\bm{V}^\top$ with orthogonal $\bm{U}$ and $\bm{V}$, then
\begin{equation}
    \bm{W}
    =
    (\bm{U}\bm{U}_0)\bm{\Sigma}(\bm{V}\bm{V}_0)^\top ,
\end{equation}
which is a valid SVD with the same singular values as $\bm{W}_0$. Thus $\bm{W}\in\mathcal{M}_{\bm{W}_0}$.
\end{proof}

\begin{lemma}[Tangent Space of the Isospectral Manifold]
For any $\bm{W}\in \mathcal{M}_{\bm{W}_0}$, the tangent space is
\begin{equation}
    T_{\bm{W}}\mathcal{M}_{\bm{W}_0}
    =
    \left\{
    \bm{G}_{\mathrm{out}}\bm{W}
    +
    \bm{W}\bm{G}_{\mathrm{in}}
    \mid
    \bm{G}_{\mathrm{out}}\in \mathfrak{so}(m),\,
    \bm{G}_{\mathrm{in}}\in \mathfrak{so}(n)
    \right\},
\end{equation}
where $\mathfrak{so}(k)=\{\bm{G}\in\mathbb{R}^{k\times k}\mid \bm{G}^\top=-\bm{G}\}$.
\label{lemma:tangent-space-isospectral}
\end{lemma}

\begin{lemma}[First-Order Stationarity on $\mathcal{M}_{\bm{W}_0}$]
Let $f:\mathbb{R}^{m\times n}\to \mathbb{R}$ be smooth and let $\bm{G}=\nabla f(\bm{W})$. A point $\bm{W}\in \mathcal{M}_{\bm{W}_0}$ is a first-order critical point of $f$ restricted to $\mathcal{M}_{\bm{W}_0}$ if and only if
\begin{equation}
    \bm{G}_{\mathrm{in}}
    =
    \bm{W}^\top\bm{G}-\bm{G}^\top\bm{W}
    =
    \bm{0},
    \qquad
    \bm{G}_{\mathrm{out}}
    =
    \bm{G}\bm{W}^\top-\bm{W}\bm{G}^\top
    =
    \bm{0}.
\end{equation}
\label{lemma:first-order-condition-pion}
\end{lemma}

\begin{proof}
By the first-order optimality condition on a smooth embedded manifold, $\bm{W}$ is stationary if and only if
\begin{equation}
    \langle \bm{G}, \delta \bm{W}\rangle = 0,
    \qquad
    \forall \delta \bm{W}\in T_{\bm{W}}\mathcal{M}_{\bm{W}_0}.
\end{equation}
Using Lemma~\ref{lemma:tangent-space-isospectral}, any tangent vector can be written as
$\delta \bm{W}=\bm{G}_{m}\bm{W}+\bm{W}\bm{G}_{n}$, where
$\bm{G}_{m}\in\mathfrak{so}(m)$ and
$\bm{G}_{n}\in\mathfrak{so}(n)$. Therefore,
\begin{equation}
    \operatorname{Tr}\!\left(
    \bm{G}^\top
    (
    \bm{G}_{m}\bm{W}
    +
    \bm{W}\bm{G}_{n}
    )
    \right)
    =
    0
\end{equation}
for all skew-symmetric $\bm{G}_{m}$ and $\bm{G}_{n}$.

Since $\bm{G}_{m}$ and $\bm{G}_{n}$ vary independently, this is equivalent to
\begin{equation}
    \operatorname{Tr}(\bm{W}\bm{G}^\top\bm{G}_{m})=0,
    \quad
    \forall \bm{G}_{m}\in\mathfrak{so}(m),
\end{equation}
and
\begin{equation}
    \operatorname{Tr}(\bm{G}^\top\bm{W}\bm{G}_{n})=0,
    \quad
    \forall \bm{G}_{n}\in\mathfrak{so}(n).
\end{equation}
The orthogonal complement of skew-symmetric matrices under the Frobenius inner product is the space of symmetric matrices. Hence $\bm{W}\bm{G}^\top$ and $\bm{G}^\top\bm{W}$ must be symmetric, which gives
\begin{equation}
    \bm{G}\bm{W}^\top-\bm{W}\bm{G}^\top=\bm{0},
    \qquad
    \bm{W}^\top\bm{G}-\bm{G}^\top\bm{W}=\bm{0}.
\end{equation}
This proves the claim.
\end{proof}

Lemma~\ref{lemma:first-order-condition-pion} shows that it is sufficient to prove
\begin{equation}
    \|\bm{G}_t^{\mathrm{in}}\|_F\to 0,
    \qquad
    \|\bm{G}_t^{\mathrm{out}}\|_F\to 0.
\end{equation}
For the bilateral update, we therefore use the combined stationarity measure
\begin{equation}
    \mathcal{S}_t
    =
    \|\bm{G}_t^{\mathrm{in}}\|_F^2
    +
    \|\bm{G}_t^{\mathrm{out}}\|_F^2.
\end{equation}
Showing $\mathcal{S}_t\to 0$ directly implies first-order convergence on the isospectral manifold.

\begin{assumption}[$L$-smoothness]
\label{assumption:l-smooth}
The objective function $f$ is $L$-smooth in the Euclidean sense, i.e.,
\begin{equation}
    f(\bm{W}_2)
    \le
    f(\bm{W}_1)
    +
    \langle \nabla f(\bm{W}_1),\bm{W}_2-\bm{W}_1\rangle
    +
    \frac{L}{2}\|\bm{W}_2-\bm{W}_1\|_F^2.
\end{equation}
\end{assumption}

\begin{assumption}[Lower boundedness]
\label{assumption:lower-bounded}
There exists $f_{\inf}\in\mathbb{R}$ such that $f(\bm{W})\ge f_{\inf}$ for all $\bm{W}$.
\end{assumption}

\begin{assumption}[Boundedness along the trajectory]
\label{assumption:bounded-trajectory}
Along the trajectory, there exist constants $\gamma,G,B>0$ such that
\begin{equation}
    \|\bm{W}_t\|_2\le \gamma,
    \qquad
    \|\nabla f(\bm{W}_t)\|_F\le G,
    \qquad
    \|\bm{G}_t^{\mathrm{in}}\|_F,\,
    \|\bm{G}_t^{\mathrm{out}}\|_F\le B.
\end{equation}
\end{assumption}

Assumptions~\ref{assumption:l-smooth} and~\ref{assumption:lower-bounded} are standard in first-order convergence analysis. Assumption~\ref{assumption:bounded-trajectory} is mild in our setting because the ideal spectrum-preserving dynamics remains on a compact isospectral manifold, and the truncated update stays in a bounded neighborhood for sufficiently small step size.

\subsection{Simplified Single-Side Analysis}

We first briefly revisit the single-side update, since it provides the key descent identity used later for the bilateral update. Consider the in-side update
\begin{equation}
    \bm{W}_{t+1}
    =
    \bm{W}_t\exp^{(2)}(-\eta\bm{G}_t^{\mathrm{in}}),
    \qquad
    \bm{G}_t^{\mathrm{in}}
    =
    \bm{W}_t^\top\bm{G}_t-\bm{G}_t^\top\bm{W}_t.
\end{equation}
Expanding the truncated exponential gives
\begin{equation}
    \bm{W}_{t+1}-\bm{W}_t
    =
    -\eta \bm{W}_t\bm{G}_t^{\mathrm{in}}
    +
    \frac{\eta^2}{2}\bm{W}_t(\bm{G}_t^{\mathrm{in}})^2.
\end{equation}
Let $\bm{S}_t=\bm{W}_t^\top\bm{G}_t$. Since
$\bm{G}_t^{\mathrm{in}}=\bm{S}_t-\bm{S}_t^\top$, we have
\begin{equation}
    \left\langle
    \bm{G}_t,
    \bm{W}_t\bm{G}_t^{\mathrm{in}}
    \right\rangle
    =
    \operatorname{Tr}(\bm{G}_t^\top\bm{W}_t\bm{G}_t^{\mathrm{in}})
    =
    \frac{1}{2}
    \|\bm{G}_t^{\mathrm{in}}\|_F^2.
\end{equation}
Therefore, the first-order part of the update gives the descent term
\begin{equation}
    \left\langle
    \bm{G}_t,
    -\eta\bm{W}_t\bm{G}_t^{\mathrm{in}}
    \right\rangle
    =
    -\frac{\eta}{2}
    \|\bm{G}_t^{\mathrm{in}}\|_F^2.
\end{equation}
The out-side update gives the analogous identity
\begin{equation}
    \left\langle
    \bm{G}_t,
    -\eta\bm{G}_t^{\mathrm{out}}\bm{W}_t
    \right\rangle
    =
    -\frac{\eta}{2}
    \|\bm{G}_t^{\mathrm{out}}\|_F^2.
\end{equation}
Thus, both in-side and out-side rotations are aligned with the Riemannian descent direction. The remaining second-order terms can be controlled by smoothness and boundedness, yielding convergence for the alternating version. Since the bilateral update simply applies both descent directions in the same step, we next analyze it directly.

\subsection{Deterministic Bilateral Update}
\label{app:deterministic-bilateral}
We consider the simplified bilateral Pion update
\begin{equation}
    \bm{W}_{t+1}
    =
    \exp^{(2)}(-\eta\bm{G}_t^{\mathrm{out}})
    \bm{W}_t
    \exp^{(2)}(-\eta\bm{G}_t^{\mathrm{in}}).
\end{equation}
Expanding the update gives
\begin{equation}
    \bm{W}_{t+1}-\bm{W}_t
    =
    -\eta
    \left(
    \bm{G}_t^{\mathrm{out}}\bm{W}_t
    +
    \bm{W}_t\bm{G}_t^{\mathrm{in}}
    \right)
    +
    \bm{R}_t,
\end{equation}
where the remainder $\bm{R}_t$ contains all terms of order $\eta^2$ and higher:
\begin{align}
    \bm{R}_t
    =
    &\frac{\eta^2}{2}(\bm{G}_t^{\mathrm{out}})^2\bm{W}_t
    +
    \eta^2 \bm{G}_t^{\mathrm{out}}\bm{W}_t\bm{G}_t^{\mathrm{in}}
    +
    \frac{\eta^2}{2}\bm{W}_t(\bm{G}_t^{\mathrm{in}})^2
    \nonumber\\
    &-
    \frac{\eta^3}{2}(\bm{G}_t^{\mathrm{out}})^2\bm{W}_t\bm{G}_t^{\mathrm{in}}
    -
    \frac{\eta^3}{2}\bm{G}_t^{\mathrm{out}}\bm{W}_t(\bm{G}_t^{\mathrm{in}})^2
    +
    \frac{\eta^4}{4}(\bm{G}_t^{\mathrm{out}})^2
    \bm{W}_t
    (\bm{G}_t^{\mathrm{in}})^2 .
\end{align}
Using Assumption~\ref{assumption:bounded-trajectory}, there exists a constant $K_\eta=O(\eta^2)$ such that
\begin{equation}
    \|\bm{R}_t\|_F
    \le
    K_\eta \mathcal{S}_t,
    \qquad
    K_\eta
    =
    \gamma
    \left(
    \eta^2+\frac{B\eta^3}{2}+\frac{B^2\eta^4}{8}
    \right).
\end{equation}

By $L$-smoothness,
\begin{equation}
    f(\bm{W}_{t+1})-f(\bm{W}_t)
    \le
    \langle \bm{G}_t,\bm{W}_{t+1}-\bm{W}_t\rangle
    +
    \frac{L}{2}
    \|\bm{W}_{t+1}-\bm{W}_t\|_F^2.
\end{equation}
Substituting the bilateral expansion and using the descent identities from the single-side analysis, we obtain
\begin{equation}
    \left\langle
    \bm{G}_t,
    -\eta
    (
    \bm{G}_t^{\mathrm{out}}\bm{W}_t
    +
    \bm{W}_t\bm{G}_t^{\mathrm{in}}
    )
    \right\rangle
    =
    -\frac{\eta}{2}\mathcal{S}_t.
\end{equation}
The remainder satisfies
\begin{equation}
    \langle \bm{G}_t,\bm{R}_t\rangle
    \le
    \|\bm{G}_t\|_F\|\bm{R}_t\|_F
    \le
    G K_\eta \mathcal{S}_t.
\end{equation}
For the quadratic term,
\begin{align}
    \|\bm{W}_{t+1}-\bm{W}_t\|_F^2
    &\le
    2\eta^2
    \left\|
    \bm{G}_t^{\mathrm{out}}\bm{W}_t
    +
    \bm{W}_t\bm{G}_t^{\mathrm{in}}
    \right\|_F^2
    +
    2\|\bm{R}_t\|_F^2
    \nonumber\\
    &\le
    4\eta^2\gamma^2\mathcal{S}_t
    +
    2K_\eta^2\mathcal{S}_t^2
    \nonumber\\
    &\le
    \left(
    4\eta^2\gamma^2
    +
    4K_\eta^2B^2
    \right)
    \mathcal{S}_t,
\end{align}
where the last inequality uses $\mathcal{S}_t\le 2B^2$. Combining the above inequalities yields
\begin{equation}
    f(\bm{W}_{t+1})-f(\bm{W}_t)
    \le
    -c_\eta \mathcal{S}_t,
\end{equation}
where
\begin{equation}
    c_\eta
    =
    \frac{\eta}{2}
    -
    G K_\eta
    -
    2L\eta^2\gamma^2
    -
    2LK_\eta^2B^2.
\end{equation}
Since $K_\eta=O(\eta^2)$, we have $c_\eta>0$ for sufficiently small $\eta$.

Summing over $t=0,\ldots,T-1$ gives
\begin{equation}
    f(\bm{W}_T)-f(\bm{W}_0)
    \le
    -c_\eta
    \sum_{t=0}^{T-1}\mathcal{S}_t.
\end{equation}
Using $f(\bm{W}_T)\ge f_{\inf}$, we obtain
\begin{equation}
    \sum_{t=0}^{T-1}\mathcal{S}_t
    \le
    \frac{f(\bm{W}_0)-f_{\inf}}{c_\eta}.
\end{equation}
Therefore,
\begin{equation}
    \min_{0\le t<T}
    \left(
    \|\bm{G}_t^{\mathrm{in}}\|_F^2
    +
    \|\bm{G}_t^{\mathrm{out}}\|_F^2
    \right)
    \le
    \frac{f(\bm{W}_0)-f_{\inf}}{c_\eta T}.
\end{equation}
By Lemma~\ref{lemma:first-order-condition-pion}, this implies convergence to a first-order stationary point on the isospectral manifold.

\begin{theorem}[Deterministic Convergence of Simplified Bilateral Pion]
\label{thm:deterministic-bilateral-pion}
Under Assumptions~\ref{assumption:l-smooth}--\ref{assumption:bounded-trajectory}, suppose the learning rate $\eta$ is sufficiently small such that $c_\eta>0$. Then the simplified bilateral Pion update satisfies
\begin{equation}
    \min_{0\le t<T}
    \left(
    \|\bm{G}_t^{\mathrm{in}}\|_F^2
    +
    \|\bm{G}_t^{\mathrm{out}}\|_F^2
    \right)
    \le
    \frac{f(\bm{W}_0)-f_{\inf}}{c_\eta T}.
\end{equation}
In particular, the stationarity measure converges at rate $\mathcal{O}(\frac{1}{T})$.
\end{theorem}

\subsection{Ambient-Space Behavior of the Deterministic Update}
\label{app:ambient-gradient}

Theorems~\ref{thm:bilateral-pion-convergence} and~\ref{thm:deterministic-bilateral-pion} control the Lie-algebra gradients, which measure stationarity on $\mathcal{M}_{\bm{W}_0}$ only. This subsection records what the deterministic analysis implies for the ambient gradient $\nabla f(\bm{W}_t)$. Here $L$-smoothness is understood in the usual Lipschitz-gradient sense, $\|\nabla f(\bm{X})-\nabla f(\bm{Y})\|_F\le L\|\bm{X}-\bm{Y}\|_F$, which implies the quadratic upper bound in Assumption~\ref{assumption:l-smooth}.

\begin{proposition}[Ambient-space behavior of deterministic Pion]
\label{prop:ambient-gradient}
Consider the simplified deterministic bilateral update of Appendix~\ref{app:deterministic-bilateral} under Assumptions~\ref{assumption:lower-bounded} and~\ref{assumption:bounded-trajectory}. Assume $\nabla f$ is $L$-Lipschitz and $\eta$ is small enough that $c_\eta>0$, and let $\mathcal{S}_t=\|\bm{G}^{\mathrm{in}}_t\|_F^2+\|\bm{G}^{\mathrm{out}}_t\|_F^2$. Then
\begin{enumerate}[leftmargin=*]
\item $\sum_{t=0}^{\infty}\mathcal{S}_t\le \dfrac{f(\bm{W}_0)-f_{\inf}}{c_\eta}<\infty$, and hence $\mathcal{S}_t\to 0$.
\item $\|\nabla f(\bm{W}_{t+1})-\nabla f(\bm{W}_t)\|_F\le L\big(\sqrt{2}\,\eta\gamma\sqrt{\mathcal{S}_t}+K_\eta\mathcal{S}_t\big)\to 0$.
\item The gradient itself need not vanish: there are instances satisfying all the assumptions with $\mathcal{S}_t=0$ for all $t$ but $\|\nabla f(\bm{W}_t)\|_F\neq 0$ for all $t$ (Remark~\ref{ex:ambient-nonzero}).
\end{enumerate}
\end{proposition}

\begin{proof}
(i) The descent inequality $f(\bm{W}_{t+1})-f(\bm{W}_t)\le -c_\eta\mathcal{S}_t$ was established in the proof of Theorem~\ref{thm:deterministic-bilateral-pion}. Summing over $t=0,\ldots,T-1$ and using $f(\bm{W}_T)\ge f_{\inf}$ gives $\sum_{t=0}^{T-1}\mathcal{S}_t\le (f(\bm{W}_0)-f_{\inf})/c_\eta$ for every $T$. Letting $T\to\infty$ proves (i).

(ii) The expansion in Appendix~\ref{app:deterministic-bilateral} gives $\bm{W}_{t+1}-\bm{W}_t=-\eta(\bm{G}^{\mathrm{out}}_t\bm{W}_t+\bm{W}_t\bm{G}^{\mathrm{in}}_t)+\bm{R}_t$ with $\|\bm{R}_t\|_F\le K_\eta\mathcal{S}_t$. Since $\|\bm{W}_t\|_2\le\gamma$,
\[
\|\bm{G}^{\mathrm{out}}_t\bm{W}_t+\bm{W}_t\bm{G}^{\mathrm{in}}_t\|_F\le\gamma\big(\|\bm{G}^{\mathrm{out}}_t\|_F+\|\bm{G}^{\mathrm{in}}_t\|_F\big)\le\sqrt{2}\,\gamma\sqrt{\mathcal{S}_t},
\]
so $\|\bm{W}_{t+1}-\bm{W}_t\|_F\le\sqrt{2}\,\eta\gamma\sqrt{\mathcal{S}_t}+K_\eta\mathcal{S}_t$. The Lipschitz property of $\nabla f$ gives the bound in (ii), whose right-hand side tends to zero by (i).

(iii) This is the example in Remark~\ref{ex:ambient-nonzero}.
\end{proof}

\begin{remark}[Stochastic setting]
Part (i) relies on the deterministic telescoping. For the stochastic update, Theorem~\ref{thm:stochastic-bilateral-pion} only gives $\min_{0\le t<T}\mathbb{E}[\mathcal{S}_t]=\mathcal{O}(T^{-1/2})$, which does not imply summability of $\mathcal{S}_t$. Hence the statement that the variation of the ambient gradient vanishes is claimed only for the deterministic update.
\end{remark}

\begin{remark}[The ambient gradient need not vanish]
\label{ex:ambient-nonzero}
Let $f(\bm{W})=\tfrac12\|\bm{W}-2\bm{I}_n\|_F^2$ and $\bm{W}_0=\bm{I}_n$, so that $\mathcal{M}_{\bm{W}_0}=\mathrm{O}(n)$ by Lemma~\ref{lemma:character-isospectral-manifold}. The gradient $\nabla f(\bm{W})=\bm{W}-2\bm{I}_n$ is 1-Lipschitz and $f\ge 0$. At $\bm{W}=\bm{I}_n$ we have $\nabla f=-\bm{I}_n$, so $\bm{G}^{\mathrm{in}}=\bm{W}^\top\nabla f-\nabla f^\top\bm{W}=\bm{0}$ and $\bm{G}^{\mathrm{out}}=\bm{0}$. Hence $\bm{W}_1=\exp^{(2)}(\bm{0})\bm{I}_n\exp^{(2)}(\bm{0})=\bm{I}_n$, and by induction $\bm{W}_t=\bm{I}_n$, $\mathcal{S}_t=0$ and $\|\nabla f(\bm{W}_t)\|_F=\sqrt{n}$ for all $t$. By Lemma~\ref{lemma:first-order-condition-pion}, $\bm{W}_0$ is a first-order critical point of $f$ on $\mathcal{M}_{\bm{W}_0}$. It is in fact the global minimizer there: for $\bm{W}\in\mathrm{O}(n)$ we have $\|\bm{W}\|_F^2=n$, so $f(\bm{W})=\tfrac12(5n-4\operatorname{tr}\bm{W})$, and $\operatorname{tr}\bm{W}\le n$ with equality only at $\bm{W}=\bm{I}_n$, since the eigenvalues of an orthogonal matrix lie on the unit circle. Nevertheless $\nabla f(\bm{W}_0)\neq\bm{0}$, so $\bm{W}_0$ is not a critical point of $f$ on the full space. The unconstrained minimizer $2\bm{I}_n$ has singular values $2\neq 1$, so it lies outside $\mathcal{M}_{\bm{W}_0}$ and is never reached.
\end{remark}

\subsection{Stochastic Bilateral Update}

We now consider the stochastic setting. Let
\begin{equation}
    \tilde{\bm{G}}_t
    =
    \bm{G}_t+\bm{\xi}_t
\end{equation}
be an unbiased mini-batch gradient estimator, where
$\mathbb{E}_t[\bm{\xi}_t]=\bm{0}$ and
$\mathbb{E}_t[\|\bm{\xi}_t\|_F^2]\le \sigma^2$.
Define
\begin{equation}
    \tilde{\bm{G}}_t^{\mathrm{in}}
    =
    \bm{W}_t^\top\tilde{\bm{G}}_t
    -
    \tilde{\bm{G}}_t^\top\bm{W}_t,
    \qquad
    \tilde{\bm{G}}_t^{\mathrm{out}}
    =
    \tilde{\bm{G}}_t\bm{W}_t^\top
    -
    \bm{W}_t\tilde{\bm{G}}_t^\top .
\end{equation}
By linearity and unbiasedness,
\begin{equation}
    \mathbb{E}_t[\tilde{\bm{G}}_t^{\mathrm{in}}]
    =
    \bm{G}_t^{\mathrm{in}},
    \qquad
    \mathbb{E}_t[\tilde{\bm{G}}_t^{\mathrm{out}}]
    =
    \bm{G}_t^{\mathrm{out}}.
\end{equation}
Moreover, the induced stochastic noise satisfies
\begin{align}
    &\mathbb{E}_t
    \left[
    \|\tilde{\bm{G}}_t^{\mathrm{in}}-\bm{G}_t^{\mathrm{in}}\|_F^2
    +
    \|\tilde{\bm{G}}_t^{\mathrm{out}}-\bm{G}_t^{\mathrm{out}}\|_F^2
    \right]
    \nonumber\\
    &\qquad\le
    8\gamma^2\sigma^2
    \triangleq
    \sigma_\Omega^2.
\end{align}
Therefore,
\begin{equation}
    \mathbb{E}_t
    \left[
    \|\tilde{\bm{G}}_t^{\mathrm{in}}\|_F^2
    +
    \|\tilde{\bm{G}}_t^{\mathrm{out}}\|_F^2
    \right]
    \le
    \mathcal{S}_t+\sigma_\Omega^2.
\end{equation}

Applying the deterministic bilateral descent argument to the stochastic update and taking conditional expectation gives
\begin{equation}
    \mathbb{E}_t[f(\bm{W}_{t+1})]-f(\bm{W}_t)
    \le
    -\frac{\eta}{2}\mathcal{S}_t
    +
    a_\eta
    (\mathcal{S}_t+\sigma_\Omega^2),
\end{equation}
where
\begin{equation}
    a_\eta
    =
    GK_\eta
    +
    2L\eta^2\gamma^2
    +
    2LK_\eta^2B^2.
\end{equation}
Equivalently,
\begin{equation}
    \mathbb{E}_t[f(\bm{W}_{t+1})]-f(\bm{W}_t)
    \le
    -c_\eta \mathcal{S}_t
    +
    a_\eta\sigma_\Omega^2,
    \qquad
    c_\eta=\frac{\eta}{2}-a_\eta.
\end{equation}
Taking total expectation and summing over $t=0,\ldots,T-1$, we obtain
\begin{equation}
    c_\eta
    \sum_{t=0}^{T-1}
    \mathbb{E}[\mathcal{S}_t]
    \le
    f(\bm{W}_0)-f_{\inf}
    +
    T a_\eta\sigma_\Omega^2.
\end{equation}
Thus,
\begin{equation}
    \min_{0\le t<T}
    \mathbb{E}[\mathcal{S}_t]
    \le
    \frac{f(\bm{W}_0)-f_{\inf}}{c_\eta T}
    +
    \frac{a_\eta}{c_\eta}\sigma_\Omega^2.
\end{equation}

Choosing $\eta=C/\sqrt{T}$ with sufficiently small $C$ gives the standard stochastic nonconvex rate
\begin{equation}
    \min_{0\le t<T}
    \mathbb{E}
    \left[
    \|\bm{G}_t^{\mathrm{in}}\|_F^2
    +
    \|\bm{G}_t^{\mathrm{out}}\|_F^2
    \right]
    =
    \mathcal{O}\!\left(\frac{1}{\sqrt{T}}\right).
\end{equation}

\begin{theorem}[Stochastic Convergence of Simplified Bilateral Pion]
\label{thm:stochastic-bilateral-pion}
Under Assumptions~\ref{assumption:l-smooth}--\ref{assumption:bounded-trajectory}, assume the stochastic gradient is unbiased and has bounded variance. Let $\eta=C/\sqrt{T}$ with sufficiently small $C>0$. Then
\begin{equation}
    \min_{0\le t<T}
    \mathbb{E}
    \left[
    \|\bm{G}_t^{\mathrm{in}}\|_F^2
    +
    \|\bm{G}_t^{\mathrm{out}}\|_F^2
    \right]
    =
    \mathcal{O}\left(\frac{1}{\sqrt{T}}\right).
\end{equation}
By Lemma~\ref{lemma:first-order-condition-pion}, this implies convergence to a stochastic first-order stationary neighborhood on the isospectral manifold.
\end{theorem}

\newpage

\section{Additional Discussion on Computation Complexity}
\label{sec: computation overhead}

\textbf{Alternate Update}. When alternate update is enabled, only one exponential map is applied at each step. On an output-side step, the update-side cost becomes
$
2d_{\mathrm{out}}^2d_{\mathrm{in}}
$
for RMS scaling,
$
\mathcal{O}(d_{\mathrm{out}}^3)
$
for forming the second-order term, and
$
2d_{\mathrm{out}}^2d_{\mathrm{in}}
$
for multiplying the exponential approximation with $\bm{W}$.
On an input-side step, the corresponding cost is
$
2d_{\mathrm{out}}d_{\mathrm{in}}^2
+
\mathcal{O}(d_{\mathrm{in}}^3)
+
2d_{\mathrm{out}}d_{\mathrm{in}}^2
$.
Averaged over two consecutive steps, the dominant update-side cost is therefore reduced from
$
4d_{\mathrm{out}}^2d_{\mathrm{in}}
+
4d_{\mathrm{out}}d_{\mathrm{in}}^2
+
\mathcal{O}(d_{\mathrm{out}}^3+d_{\mathrm{in}}^3)
$
to approximately
$
2d_{\mathrm{out}}^2d_{\mathrm{in}}
+
2d_{\mathrm{out}}d_{\mathrm{in}}^2
+
\mathcal{O}\!\left(\frac{d_{\mathrm{out}}^3+d_{\mathrm{in}}^3}{2}\right).
$
Thus, alternate update reduces the update-side computation by roughly a factor of two.

\textbf{Memory Analysis}.
The persistent optimizer states consist of the first- and second-moment buffers on both Lie algebras:
$\bm{m}^{\mathrm{in}},\bm{v}^{\mathrm{in}}\in\mathbb{R}^{d_{\mathrm{in}}\times d_{\mathrm{in}}}$
and
$\bm{m}^{\mathrm{out}},\bm{v}^{\mathrm{out}}\in\mathbb{R}^{d_{\mathrm{out}}\times d_{\mathrm{out}}}$.
Thus, the persistent memory overhead is
$
2(d_{\mathrm{in}}^2+d_{\mathrm{out}}^2)
$
floating-point numbers, excluding the parameter matrix and its backpropagated gradient. In comparison, Adam-style optimizers store
$
2d_{\mathrm{out}}d_{\mathrm{in}}
$
floating-point numbers for the same weight matrix. Therefore, the relative persistent-state overhead of Pion compared with Adam is
$
(d_{\mathrm{out}}^2+d_{\mathrm{in}}^2)/(d_{\mathrm{out}}d_{\mathrm{in}})
$.
For nearly square matrices this is a small constant factor, while for highly rectangular matrices it increases with the aspect ratio.

The transient memory is dominated by the Lie algebra gradients, the normalized algebra directions, the squared Lie matrices used in $\mathcal{E}_2$, and temporary products with $\bm{W}$. With buffer reuse or in-place construction of $\bm{A}^{\mathrm{in}}$ and $\bm{A}^{\mathrm{out}}$, the additional peak temporary memory is
$
\mathcal{O}(d_{\mathrm{in}}^2+d_{\mathrm{out}}^2+d_{\mathrm{out}}d_{\mathrm{in}})
$.
A naive implementation that materializes both sides of the RMS-scaling term
$\bm{A}^{\mathrm{out}}\bm{W}$ and $\bm{W}\bm{A}^{\mathrm{in}}$
may require two extra $d_{\mathrm{out}}\times d_{\mathrm{in}}$ buffers, but this can be reduced by accumulating the Frobenius norm with buffer reuse.

\textbf{Practical Cost}.
We further report the practical memory and runtime cost under the same training configuration.
All experiments are conducted on 8 NVIDIA H100 GPUs connected with NVLink. We use distributed data parallelism (DDP) without gradient accumulation, and report the peak memory usage per GPU together with the wall-clock time per training step.

\begin{table}[h]
\centering
\scriptsize
\renewcommand{\arraystretch}{1.3}
\setlength{\tabcolsep}{13pt}
\vspace{3mm}
\begin{tabular}{l|c|c|c|c}
Method
& \textbf{AdamW}
& \textbf{Muon}
& \textbf{Pion}
& \textbf{Pion w/o 2nd Moment} \\
\shline
Memory Usage per GPU & 51593 MB & 47251 MB & 59839 MB & 45289 MB \\
Time Usage per Step & 0.3932 s & 0.5505 s & 0.5679 s & 0.5678 s \\
\end{tabular}
\caption{
\scriptsize Practical memory and runtime cost of AdamW, Muon, and Pion.
All experiments are run on 8 H100 GPUs with NVLink using DDP, without gradient accumulation.
Memory is reported as peak per-GPU memory usage.
}
\label{tab:practical_cost}
\vspace{2mm}
\end{table}

The measured results are consistent with the theoretical analysis. Full Pion uses 59,839 MB per GPU, which is about 16.0\% higher than AdamW and 26.6\% higher than Muon. This increase mainly comes from the additional input- and output-side Lie algebra states and temporary matrix products. However, the variant without second-moment buffers reduces the memory usage to 45289 MB per GPU, which is 24.3\% lower than full Pion, 12.2\% lower than AdamW, and 4.2\% lower than Muon. This shows that the second-moment Lie algebra buffers are the major source of Pion's persistent memory overhead. Fortunately, Pion without second-order moment achieves only slightly worse performance compared to the full Pion.

In terms of runtime, Pion takes 0.5679 seconds per step, compared with 0.3932 seconds for AdamW and 0.5505 seconds for Muon. Thus, Pion is about 44.4\% slower than AdamW, but only about 3.2\% slower than Muon. Removing the second-moment buffers does not noticeably change the step time in this setting, indicating that the runtime is dominated by matrix multiplications and the second-order exponential approximation rather than element-wise moment updates. Overall, Pion introduces moderate additional runtime over Muon while providing a structured bilateral Lie-algebra update, and its memory footprint can be substantially reduced by removing the second-moment buffers.

\newpage
\section{Experimental Details}\label{app:exp}

In this section, we provide the detailed experimental setups and configurations for each experiment presented in the main manuscript.

\subsection{Experiments for Design Principles}
\label{app:exp_design}
All experiments are conducted in the Megatron-LM codebase with bf16 mixed-precision.

\paragraph{Shared settings.}
Unless otherwise specified, we adopt a second-order approximation to the matrix exponential due to its strong spectrum-preserving property. We use the T5-base tokenizer for text tokenization. We use a batch size of 512 for all experiments. We use a cosine decay schedule for the learning rate, decaying to $0.01$ times the initial value. The initial learning rate is selected from $\{1\text{e-}4,\,5\text{e-}4,\,1\text{e-}3,\,5\text{e-}3,\,1\text{e-}2\}$, and we observe that $\text{lr}=1\text{e-}3$ consistently achieves the best performance. Therefore, results in the main text are reported with this value. In implementation, the spectrum-preserving updates are applied in a per-head manner (e.g., per attention head). For the language modeling embedding and head layers, we employ AdamW ($\beta_1=0.9$, $\beta_2=0.95$ , weight decay$=0.1$) for updates. Pion is applied exclusively to two-dimensional parameter matrices, and no weight decay is imposed on these Pion-updated parameters.

\paragraph{Consistent Update.}
For bilateral normalization, independently normalizing $G_t^{\mathrm{out}}$ and $G_t^{\mathrm{in}}$ in Eq.~\eqref{eq:brb} alters the overall update norm and would require additional learning-rate tuning. 
To isolate the effect of normalization itself, we rescale the resulting Lie algebra update to approximately match the norm of the original update, thereby avoiding extra hyperparameter adjustments.

\paragraph{Momentum.}
We observe that directly applying momentum without RMS scaling is only stable under extremely small learning rates. 
Therefore, we apply RMS scaling to all variants and evaluate them under a unified setting. The RMS scaling coefficient is set to $0.2$. 
For all variants, we set the first-order momentum coefficient to $\beta_1 = 0.9$ and the second-order momentum coefficient to $\beta_2 = 0.95$.

\paragraph{Alternate Update.}
The settings follow those of the momentum experiments. 
For the Lie+Lie variant, we find that disabling momentum accumulation on the non-updated side degrades convergence, despite reducing gradient computation. This aligns with classical stochastic optimization theory, where reducing the effective sample size for first-order moment estimation increases variance~\citep{robbins1951stochastic,bottou2018optimization}. 
Therefore, even under alternating updates, we maintain momentum accumulation on both sides and alternate only the parameter updates.

\paragraph{Approximation of $\exp(\cdot)$.}
The settings follow the shared configuration above. 
For the Cayley transform, we adopt the standard formulation 
$(\bm{I} - \frac{1}{2}\bm{S})^{-1}(\bm{I} + \frac{1}{2}\bm{S})$, 
where $\mathbf{S}$ is a square matrix. 
In implementation, computations are performed in float32 for numerical stability, and the corresponding linear system is solved via $\texttt{torch.linalg.solve}$.

\subsection{Pretraining}
\label{app:exp_pretrain}

The pretraining experiments are conducted in the Megatron-LM codebase using bfloat16 precision. We use the T5-base tokenizer for text tokenization.  We follow standard settings and adopt a cosine decay schedule for the learning rate, decaying to $0.01$ times the maximum learning rate. The maximum learning rate is set to $5\text{e-}4$ for all methods, which is chosen based on the use of the AdamW optimizer. For both Muon and Pion, we set the RMS scaling coefficient to $0.2$. For optimizer hyperparameters, we use $\beta_1 = 0.9$ and $\beta_2 = 0.95$ for AdamW and Pion. For Muon, we set $\beta_1 = 0.95$. For attention updates, Muon adopts a split-head strategy, where each attention head is updated independently. For Pion, due to memory constraints, we apply updates separately to query, key and value matrices. We use a batch size of 512 for all experiments. For both Muon and Pion, the embedding layers, output head, and bias vectors are updated using AdamW. We set a unified weight decay coefficient of $0.1$ for AdamW and Muon. In contrast, for the parameters updated by Pion, we do not apply weight decay.

\subsection{Supervised Fine-tuning}
\label{app:exp sft}

\paragraph{Models and Datasets.} We employ Qwen2.5-1.5B~\citep{team2024qwen2} and Llama-3.2-3B~\citep{grattafiori2024llama} as our base models, fine-tuning them on the MetaMathQA~\citep{yu2023metamath} and Magicoder-Evol-Instruct-110K~\citep{wei2024magicoder} datasets. To maintain a uniform computational budget across both domains, we limit each dataset to 50K samples.

\paragraph{Training Details.} All training experiments are conducted within the LLaMA-Factory~\citep{zheng2024llamafactory} framework and run on NVIDIA H200 GPUs. To ensure a fair comparison, all methods are trained for 3 epochs with a global batch size of 64, a learning rate of $1 \times 10^{-5}$, a cutoff length of 4096 tokens, and using the default FP32 numerical precision for optimizer updates. Our implementation of Pion optimizer deliberately omits the second-order momentum and employ a bilateral update strategy. Furthermore, we split the attention heads and the FFN gate/up projections for separate optimization. 

\paragraph{Evaluation Protocols.} All evaluations are conducted using the LM Evaluation Harness~\citep{eval-harness} framework with its default generation parameters. Specifically, we evaluate GSM8K~\citep{cobbe2021training} in a 5-shot setting, while all other benchmarks are evaluated zero-shot. For generation-based tasks, we apply deterministic greedy decoding. Performance on the HumanEval~\citep{chen2021evaluating} benchmark is reported using the pass@1 metric. 

\subsection{Reinforcement Learning with Verifiable Reward}
\label{app:exp rl}

\paragraph{Models and Datasets.} Experiments are conducted on two base models, Qwen3-1.7B~\citep{yang2025qwen3} and DeepSeek-R1-Distill-Qwen-1.5B~\citep{guo2025deepseek}, using DeepMath~\citep{he2025deepmath} as the training dataset. The maximum context length is set to 4096 for Qwen3-1.7B and 8192 for DeepSeek-R1-Distill-Qwen-1.5B. 

\paragraph{Training Details.} We implement our RLVR training pipeline using the VeRL~\citep{sheng2024hybridflow} framework, utilizing vLLM~\citep{kwon2023efficient} for efficient rollouts and adopting GRPO~\citep{shao2024deepseekmath} algorithm. All experiments run on NVIDIA H200 GPUs. To ensure a fair comparison across all methods, we fix the learning rate at $1 \times 10^{-6}$, sample 12 rollouts per prompt, and use the default FP32 numerical precision for optimizer updates. Regarding model-specific configurations, Qwen3-1.7B is trained for 400 steps with both the global batch size and on-policy minibatch size set to 128. In contrast, DeepSeek-R1-Distill-Qwen-1.5B is trained for 781 steps, with both batch sizes configured to 64. Our Pion implementation operates without second-order momentum and uses an alternating update strategy. Furthermore, we split the attention heads and the FFN gate/up projections for separate optimization.

\paragraph{Evaluation Protocols.} We evaluate our models on five widely used benchmarks: AIME24~\citep{AIME24}, AIME25~\citep{AIME25}, AMC23~\citep{AMC23}, Minerva Math~\citep{lewkowycz2022solving}, and OlympiadBench~\citep{he2024olympiadbench}. The maximum generation length is set to 4,096 tokens for Qwen3-1.7B and 8,192 tokens for DeepSeek-R1-Distill-Qwen-1.5B. To ensure robust evaluation, we report the averaged accuracy across multiple independent samples for each benchmark. Specifically, we average the accuracy over 32 samples for AIME24 and AIME25, 8 samples for AMC23 and OlympiadBench, and 4 samples for Minerva Math. All evaluations are conducted within the POLARIS~\citep{Polaris2025} framework with a decoding temperature of 1.0, top-$k$ sampling with $k=20$, and top-$p$ sampling with $p=0.8$. 

\section{Limitations}
\label{app: limitaion}

While Pion demonstrates exceptional stability and competitive performance, it introduces certain computational and memory overheads. First, computing the Lie-algebra gradients and applying the truncated matrix exponential mapping requires additional FLOPs. However, this overhead is largely amortized in standard LLM pretraining regimes where large token batches are used, making the practical impact on wall-clock time minimal. Second, accumulating momentum directly in the Lie algebra moderately increases the optimizer's memory footprint. Nevertheless, our framework is highly flexible: in strictly memory-constrained settings, components like second-order momentum can be seamlessly omitted while still retaining the core benefits of spectrum preservation. Finally, scaling Pion's empirical evaluation to larger models remains an important direction for future work.

\newpage
\section{Compatibility with Maximal Update Parametrization}\label{app:mup}

In this section, we study whether the maximal update parametrization ($\mu$P) framework~\citep{yang2021tuning,yang2020feature} holds under the Pion optimizer. ~\citep{yang2023spectral} shows that $\mu$P can be expressed through strict spectral norm constraints on the weight matrices and their updates: $\|\bm{W}\|_2 = \Theta\!\left(\sqrt{\frac{d_{\mathrm{out}}}{d_{\mathrm{in}}}}\right)$ and $\|\Delta \bm{W}\|_2 = \Theta\!\left(\sqrt{\frac{d_{\mathrm{out}}}{d_{\mathrm{in}}}}\right)$, where $\Delta \bm{W}$ is the weight increment induced by a single optimization step. We refer to these as the Forward Spectral Condition and the Update Spectral Condition, respectively.

In standard optimization frameworks, both spectral conditions require active maintenance throughout training. The Pion optimizer, however, operates via left and right orthogonal transformations on the weight matrix. This unique structure inherently preserves the singular values of $\bm{W}$, making its spectral norm strictly invariant. Consequently, as long as the initialization scheme properly enforces the Forward Spectral Condition, this stability holds automatically at every subsequent step. Therefore, the remaining critical task is to determine whether the update matrix $\Delta \bm{W}$ induced by Pion can satisfy the corresponding Update Spectral Condition.

To address this, recall the Pion update takes the form
\begin{equation}
    \bm{W}_{t+1} = \exp(-\eta \bm{G}_t^{\mathrm{out}})\, \bm{W}_t\, \exp(-\eta \bm{G}_t^{\mathrm{in}}).
\end{equation}

First-order Taylor expansion around $\eta=0$ yields the weight increment
\begin{equation}
    \Delta \bm{W}_t := \bm{W}_{t+1} - \bm{W}_t \approx -\eta \bigl(\bm{G}_t^{\mathrm{out}} \bm{W}_t + \bm{W}_t \bm{G}_t^{\mathrm{in}}\bigr).
\end{equation}

By applying the triangle inequality and the submultiplicative property of the spectral norm, we can bound the update magnitude as
\begin{equation}
    \|\Delta \bm{W}_t\|_2 \lesssim \eta \|\bm{W}_t\|_2 \left(\|\bm{G}_t^{\mathrm{out}}\|_2 + \|\bm{G}_t^{\mathrm{in}}\|_2\right).
\end{equation}

Since the Pion optimizer inherently preserves the Forward Spectral Condition, the norm $\|\bm{W}_t\|_2 = \Theta\!\bigl(\sqrt{d_{\mathrm{out}}/d_{\mathrm{in}}}\bigr)$ remains strictly invariant throughout training, provided it is satisfied at initialization. Consequently, fulfilling the Update Spectral Condition reduces to controlling the spectral scales of the two generator matrices. Specifically, enforcing $\|\bm{G}_t^{\mathrm{out}}\|_2 = \Theta(1)$ and $\|\bm{G}_t^{\mathrm{in}}\|_2 = \Theta(1)$ ensures that $\|\Delta \bm{W}_t\|_2$ inherits the identical $\Theta$-order as $\|\bm{W}_t\|_2$, thereby satisfying the $\mu$P requirements. To practically enforce this $\Theta(1)$ spectral norm condition on the Lie-algebra generators, we consider two distinct methodological schemes:

\textbf{Scheme I: Spectral Norm Scaling.} The most straightforward approach is to directly constrain the spectral norm of the generators. By computing the maximum singular values of the projected gradients, we can explicitly normalize $\|\bm{G}_t^{\mathrm{out}}\|_2$ and $\|\bm{G}_t^{\mathrm{in}}\|_2$ to exactly $\Theta(1)$.

\textbf{Scheme II: Explicit Orthogonalization.} Alternatively, inspired by Muon's principle, we can explicitly orthogonalize the gradients on the Lie algebra. By applying an orthogonalization procedure via Newton-Schulz iteration to $\bm{G}_t^{\mathrm{out}}$ and $\bm{G}_t^{\mathrm{in}}$, we push their non-zero singular values toward $1$. This structurally guarantees the strictly bounded $\Theta(1)$ spectral norm condition while uniformly maintaining the update magnitude across all active spectral directions.

\begin{wrapfigure}{r}{0.43\linewidth}
\scriptsize
\centering
\includegraphics[width=1\linewidth]{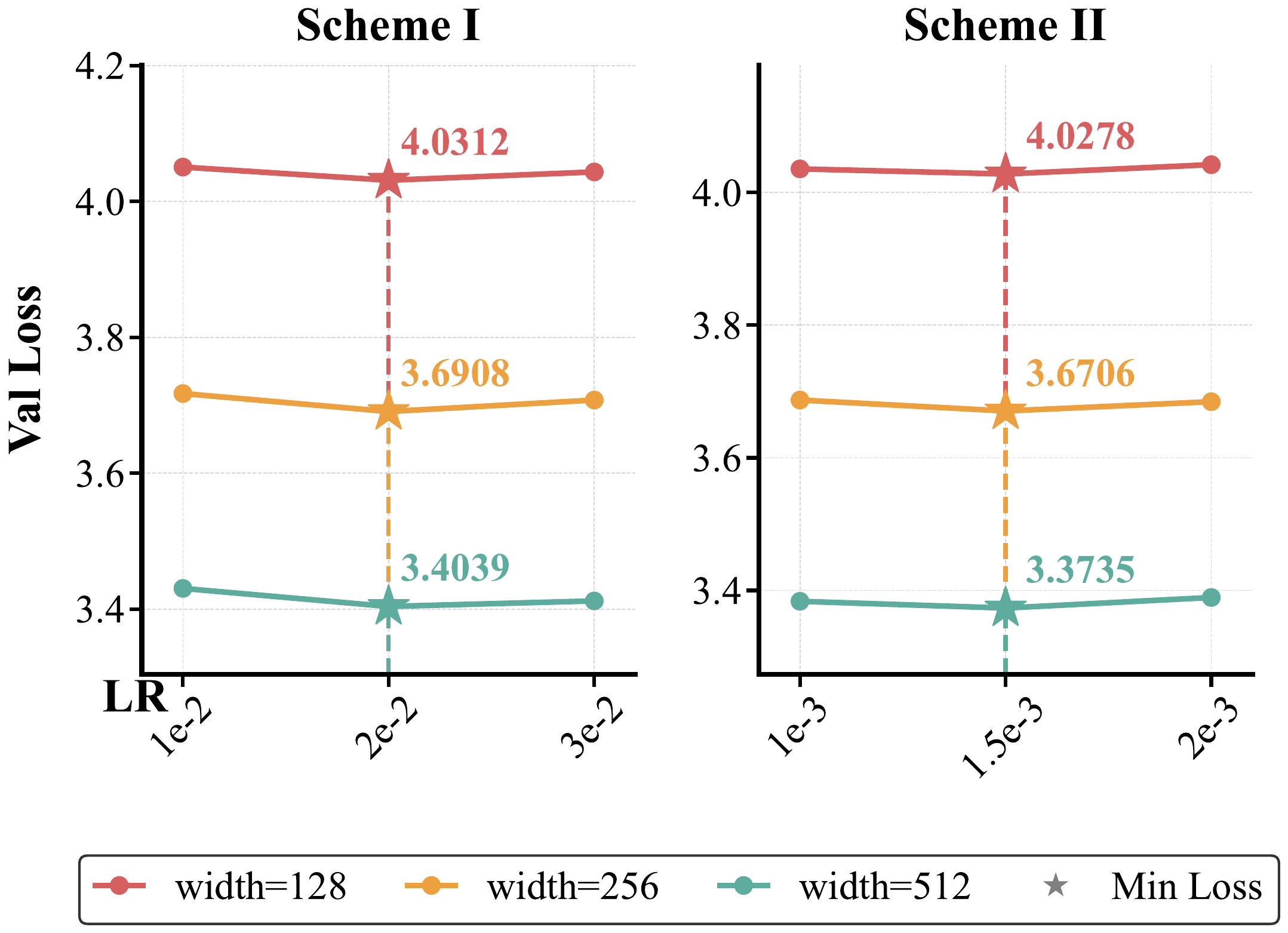}
\vspace{-3.5mm}
\caption{\scriptsize $\mu$P learning rate grid search.}
    \label{fig:mup curve}
\vspace{3mm}
\end{wrapfigure}

A key implication of satisfying the $\mu$P condition is that hyperparameters become transferable under width scaling. To verify this property for both proposed schemes, we perform experiments on a LLaMA-based architecture. Specifically, we scale the hidden size, intermediate size, and number of attention heads while keeping the head dimension fixed. For each configuration, we sweep the learning rate and identify the value that yields the lowest validation loss. The hidden size is varied over $\{128, 256, 512\}$. Regarding the specific experimental details, for Scheme I, we directly normalize the spectral norm to 1.0. For Scheme II, to prevent the large $\mu$P-derived learning rates from destabilizing the parameters optimized by AdamW, we explicitly set $\alpha = 10.0$ in Eq.~\ref{eq:alpha}. As illustrated in Fig.~\ref{fig:mup curve}, the validation loss curves across all three hidden sizes demonstrate a precise alignment of their optimal learning rates under both Scheme I and Scheme II, confirming that both approaches successfully achieve hyperparameter transferability.

\newpage
\section{Additional Results}
\subsection{Another variant of Pion}\label{app:imp_details}
Here, we provide another variant of Pion, \ie, Pion with Transported Ambient-space Momentum.
\begin{algorithm}[h!]
\LinesNumbered
\caption{Pion with Transported Ambient-Space Momentum}
\label{alg:pion_ambient}
\KwIn{Learning rate $\eta$, momentum coefficients $\beta_1,\beta_2$, RMS constant $c$, stability constant $\epsilon$, alternating flag, initial weight $\bm{W}_0 \in \mathbb{R}^{d_{\mathrm{out}}\times d_{\mathrm{in}}}$}
\KwOut{Optimized parameter $\bm{W}_t$}
Initialize $\bm{m}_0, \bm{v}_0 \leftarrow \bm{0} \in \mathbb{R}^{d_{\mathrm{out}}\times d_{\mathrm{in}}}$\;
Define $\mathcal{E}_2(\bm{A}, \alpha) \leftarrow \bm{I}+\eta\alpha \bm{A}+\frac{1}{2}(\eta\alpha \bm{A})^2$\;
\For{$t=1,2,\ldots$}{
    $\bm{G}_t \leftarrow \nabla_{\bm{W}} f(\bm{W}_{t-1})$\;
    $\bm{m}_t \leftarrow \beta_1 \bm{m}_{t-1} + (1-\beta_1)\bm{G}_t$, 
    $\bm{v}_t \leftarrow \beta_2 \bm{v}_{t-1} + (1-\beta_2)(\bm{m}_t \odot \bm{m}_t)$\;
    $\widetilde{\bm{G}}_t \leftarrow \bm{m}_t / (\sqrt{\bm{v}_t}+\epsilon)$\;
    $\bm{S}^{\mathrm{in}}_t \leftarrow \bm{W}_{t-1}^{\top}\widetilde{\bm{G}}_t - \widetilde{\bm{G}}_t^{\top}\bm{W}_{t-1}$, 
    $\bm{S}^{\mathrm{out}}_t \leftarrow \widetilde{\bm{G}}_t\bm{W}_{t-1}^{\top} - \bm{W}_{t-1}\widetilde{\bm{G}}_t^{\top}$\;
    $\bm{A}^{\mathrm{in}}_t \leftarrow -\bm{S}^{\mathrm{in}}_t$, 
    $\bm{A}^{\mathrm{out}}_t \leftarrow -\bm{S}^{\mathrm{out}}_t$\;
    \eIf{alternate update is used}{
        \eIf{$t$ is odd}{
            $\alpha_t \leftarrow c\sqrt{d_{\mathrm{out}}d_{\mathrm{in}}}/ \big(\|\bm{W}_{t-1}\bm{A}^{\mathrm{in}}_t\|_F+\epsilon\big)$\;
            $\bm{W}_t \leftarrow \bm{W}_{t-1}\mathcal{E}_2(\bm{A}^{\mathrm{in}}_t, \alpha_t)$, 
            $\bm{m}_t \leftarrow \bm{m}_t \mathcal{E}_2(\bm{A}^{\mathrm{in}}_t, \alpha_t)$\;
        }{
            $\alpha_t \leftarrow c\sqrt{d_{\mathrm{out}}d_{\mathrm{in}}}/ \big(\|\bm{A}^{\mathrm{out}}_t\bm{W}_{t-1}\|_F+\epsilon\big)$\;
            $\bm{W}_t \leftarrow \mathcal{E}_2(\bm{A}^{\mathrm{out}}_t, \alpha_t)\bm{W}_{t-1}$, 
            $\bm{m}_t \leftarrow \mathcal{E}_2(\bm{A}^{\mathrm{out}}_t, \alpha_t)\bm{m}_t$\;
        }
    }{
        $\alpha_t \leftarrow c\sqrt{d_{\mathrm{out}}d_{\mathrm{in}}}/ \big(\|\bm{A}^{\mathrm{out}}_t\bm{W}_{t-1}+\bm{W}_{t-1}\bm{A}^{\mathrm{in}}_t\|_F+\epsilon\big)$\;
        $\bm{W}_t \leftarrow \mathcal{E}_2(\bm{A}^{\mathrm{out}}_t, \alpha_t)\bm{W}_{t-1}\mathcal{E}_2(\bm{A}^{\mathrm{in}}_t, \alpha_t)$, 
        $\bm{m}_t \leftarrow \mathcal{E}_2(\bm{A}^{\mathrm{out}}_t, \alpha_t)\bm{m}_t \mathcal{E}_2(\bm{A}^{\mathrm{in}}_t, \alpha_t)$\;
    }
}
\Return{$\bm{W}_t$}\;
\end{algorithm}

\subsection{Pretraining}
\label{sec:more_results_in_pretraining}
In this section, we further provide additional results on spectrum evolution and training stability indicators. Specifically, we monitor the Frobenius norms of representative weight matrices together with their corresponding input and output activations, as well as the maximum attention logits throughout training. These quantities serve as practical indicators for characterizing optimization stability and activation amplification during large-scale training.

Figure~\ref{fig:spectrum_comparison_appendix} shows that Pion preserves the spectrum of weight matrices remarkably well throughout optimization, remaining highly consistent with the initialization. In contrast, both AdamW and Muon substantially distort the original spectra as training proceeds. This observation is consistent with the spectrum-preserving design of Pion and highlights its fundamentally different optimization dynamics.

A clear separation among optimizers can also be observed from the stability indicators (Figure~\ref{fig:attention_logit_appendix}, \ref{fig:stable_indicators_layer1_appendix},  \ref{fig:stable_indicators_layer12_appendix}, \ref{fig:stable_indicators_layer24_appendix}. AdamW exhibits continuously growing attention logits together with rapidly amplified activation magnitudes. Muon effectively suppresses the growth of attention logits, but the activations and down-projection norms still increase steadily during training. In contrast, Pion keeps nearly all monitored quantities flat and stable throughout optimization. Such distinctive behavior demonstrates the exceptional stability of Pion's spectrum-preserving updates and suggests strong potential for stable large-scale pretraining.

\begin{figure}[t!]
    \centering
    \setlength{\abovecaptionskip}{1pt}
    \setlength{\belowcaptionskip}{-6pt}
    \includegraphics[width=1\linewidth]{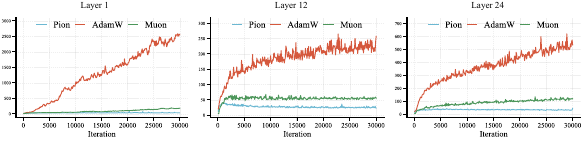}
    \caption{\scriptsize Indicators for stable pretraining. These figures show the maximum attention logit in the attention block of Layer 1, 12, 24.}
    \label{fig:attention_logit_appendix}
\end{figure}

\subsection{Supervised Fine-tuning}

To complement the aggregated results in the main text, we provide a detailed breakdown of scores for each individual benchmark in Table~\ref{tab:qwen_comparison} and Table~\ref{tab:llama32_comparison}. Furthermore, we investigate the alternating update strategy of the Pion optimizer. Empirical results indicate that while the bilateral update consistently outperforms its alternating counterpart, the alternating strategy still achieves competitive performance that is highly comparable to established baselines like AdamW and Muon. We hypothesize that in the supervised fine-tuning phase, which requires precise fitting to deterministic supervision signals, simultaneous double-side updates provide a more synchronous and direct gradient descent path.

\begin{table*}[htbp]
    \centering
    \scriptsize
    \renewcommand{\arraystretch}{1.2}
    \setlength{\tabcolsep}{2.5pt}
    \vspace{5mm}
    \begin{tabular}{ll|cccccc|ccccccc} 
        \multicolumn{2}{c|}{\multirow{2}{*}{\textbf{Method}}} & \multicolumn{6}{c|}{\textbf{MetaMath-50K}} & \multicolumn{7}{c}{\textbf{Magicoder-50K}} \\
        \multicolumn{2}{c|}{} & \makecell{ARC\_C} & \makecell{ARC\_E} & \makecell{Hella.} & \makecell{PIQA} & \makecell{Wino.} & \makecell{GSM8K} 
        & \makecell{ARC\_C} & \makecell{ARC\_E} & \makecell{Hella.} & \makecell{PIQA} & \makecell{Wino.} & \makecell{GSM8K} & \makecell{Human\\eval} \\\shline
        \multicolumn{2}{l|}{Base}  & 45.22 & 71.76 & 67.83 & 75.95 & 63.38 & 59.81 & 45.22 & 71.76 & 67.83 & 75.95 & 63.38 & 59.81 & 35.98 \\
        \multicolumn{2}{l|}{AdamW} & 41.46 & 63.46 & 66.88 & 75.35 & 63.52 & 65.88 & 43.71 & 66.62 & 67.57 & 74.64 & 64.03 & 59.28 & 51.83 \\
        \multicolumn{2}{l|}{Muon}  & 40.69 & 59.55 & 67.16 & 74.86 & 63.85 & 65.27 & 42.83 & 65.02 & 68.35 & 75.13 & 64.16 & 58.98 & 50.00 \\
        \rowcolor{gray!15} \phantom{\textbf{Pion}} & bilateral update & 42.38 & 63.59 & 66.64 & 75.80 & 62.40 & 65.76 & 44.71 & 66.29 & 66.56 & 74.59 & 63.54 & 63.54 & 53.05 \\ 
        \rowcolor{gray!15} \multirow{-2}{*}{\textbf{Pion}} & alternating update & 41.21 & 62.96 & 66.25 & 75.14 & 62.50 & 64.67 & 43.34 & 65.44 & 66.41 & 74.70 & 63.85 & 62.02 & 52.61 \\
        \specialrule{0em}{-5.25pt}{0pt}
    \end{tabular}
    \caption{\scriptsize Performance comparison of Pion and baseline optimizers on fine-tuning task using Qwen2.5-1.5B base model.}
    \label{tab:qwen_comparison}
    \vspace{6mm}
\end{table*}

\begin{table*}[htbp]
    \centering
    \scriptsize
    \renewcommand{\arraystretch}{1.2}
    \setlength{\tabcolsep}{2.5pt}
    \begin{tabular}{ll|cccccc|ccccccc} 
        \multicolumn{2}{c|}{\multirow{2}{*}{\textbf{Method}}} & \multicolumn{6}{c|}{\textbf{MetaMath-50K}} & \multicolumn{7}{c}{\textbf{Magicoder-50K}} \\
        \multicolumn{2}{c|}{} & \makecell{ARC\_C} & \makecell{ARC\_E} & \makecell{Hella} & \makecell{PIQA} & \makecell{Wino} & \makecell{GSM8K} 
        & \makecell{ARC\_C} & \makecell{ARC\_E} & \makecell{Hella} & \makecell{PIQA} & \makecell{Wino} & \makecell{GSM8K} & \makecell{Human\\eval} \\\shline
        \multicolumn{2}{l|}{Base} & 45.82 & 71.68 & 73.63 & 77.58 & 69.22 & 25.47 & 45.82 & 71.68 & 73.63 & 77.58 & 69.22 & 25.47 & 26.22 \\
        \multicolumn{2}{l|}{AdamW} & 37.03 & 57.20 & 68.49 & 76.17 & 65.43 & 59.87 & 45.39 & 69.65 & 72.02 & 76.33 & 66.06 & 22.37 & 46.95 \\
        \multicolumn{2}{l|}{Muon}  & 38.40 & 58.33 & 68.74 & 76.22 & 64.33 & 57.77 & 44.97 & 67.76 & 72.42 & 76.22 & 65.04 & 26.84 & 46.34 \\
        \rowcolor{gray!15} \phantom{\textbf{Pion}} & bilateral update & 37.09 & 56.30 & 67.70 & 76.57 & 64.54 & 58.83 & 44.14 & 68.38 & 71.94 & 76.77 & 67.69 & 29.49 & 47.19 \\
        \rowcolor{gray!15} \multirow{-2}{*}{\textbf{Pion}}& alternating update & 35.41 & 53.37 & 67.15 & 75.40 & 63.93 & 57.09 & 43.92 & 67.13 & 70.96 & 76.52 & 67.25 & 28.81 & 45.12 \\ \specialrule{0em}{-5.25pt}{0pt}
    \end{tabular}
    \caption{\scriptsize Performance comparison of Pion and baseline optimizers on fine-tuning task using Llama-3.2-3B base model.}
    \label{tab:llama32_comparison}
    \vspace{4mm}
\end{table*}

\subsection{Reinforcement Learning with Verifiable Reward}

While the main text reports the RLVR performance of AdamW, Muon, and the Pion optimizer utilizing the alternating update strategy, Table~\ref{app: tab rlvr} supplements these findings by providing the ablation results for Pion's bilateral update counterpart. Interestingly, contrary to the observations in the supervised fine-tuning phase, the alternating strategy proves superior in the RL setting. Nevertheless, the bilateral approach remains highly robust, delivering overall performance that is closely comparable to the established baselines. While yet to be rigorously verified through targeted experiments, we hypothesize that the superiority of the alternating strategy in RL stems from its implicit promotion of exploration within the parameter space. Unlike the deterministic supervision in SFT, RL relies on discovering optimal reasoning paths through sparse and noisy reward landscapes. In such settings, rigid bilateral updates might cause the policy to greedily over-optimize toward early and potentially sub-optimal reward signals, leading to premature convergence. Conversely, we conjecture that the decoupled nature of alternating updates introduces a beneficial exploratory variance into the optimization trajectory. Theoretically, this prevents the model from rapidly collapsing into local optima, thereby maintaining a healthier exploration-exploitation balance essential for robust policy learning.

\vspace{1.5em}
\begin{table*}[htbp]
    \centering
    \scriptsize
    \renewcommand{\arraystretch}{1.2}
    \setlength{\tabcolsep}{2.5pt}
    \resizebox{\textwidth}{!}{
    \begin{tabular}{ll|cccccc|cccccc} 
        \multicolumn{2}{c|}{\multirow{2}{*}{\textbf{Method}}} & \multicolumn{6}{c|}{\textbf{Qwen3-1.7B}} & \multicolumn{6}{c}{\textbf{DeepSeek-R1-Distill-Qwen-1.5B}} \\
        \multicolumn{2}{c|}{} & \makecell{AIME24 \\ (avg@32)} & \makecell{AIME25 \\ (avg@32)} & \makecell{AMC23 \\ (avg@8)} & \makecell{Minerva \\ Math \\ (avg@4)} & \makecell{Olympiad \\ Bench \\ (avg@8)} & Avg 
        & \makecell{AIME24 \\ (avg@32)} & \makecell{AIME25 \\ (avg@32)} & \makecell{AMC23 \\ (avg@8)} & \makecell{Minerva \\ Math \\ (avg@4)} & \makecell{Olympiad \\ Bench \\ (avg@8)} & Avg \\\shline
        \multicolumn{2}{l|}{Base} & 4.06 & 10.10 & 30.27 & 16.27 & 23.67 & 16.87 & 20.52 & 20.83 & 54.06 & 19.39 & 36.20 & 30.20 \\
        \multicolumn{2}{l|}{AdamW} & 22.71 & 20.94 & 58.43 & 25.91 & 46.09 & 34.82 & 25.42 & 23.94 & 62.65 & 23.16 & 44.69 & 35.97 \\
        \multicolumn{2}{l|}{Muon} & 20.42 & 19.27 & 54.22 & 24.08 & 42.41 & 32.08 & 29.06 & 23.33 & 66.72 & 22.89 & 44.61 & 37.32 \\
        \rowcolor{gray!15} \phantom{\textbf{Pion}} & bilateral update & 23.44 & 17.40 & 54.07 & 26.47 & 44.33 & 33.14 & 25.04 & 23.73 & 64.76 & 22.70 & 43.98 & 36.04 \\ 
        \rowcolor{gray!15} \multirow{-2}{*}{\textbf{Pion}} & alternate update & 25.42 & 21.98 & 59.94 & 26.84 & 46.43 & \textbf{36.12} & 30.00 & 24.38 & 66.87 & 23.90 & 46.43 & \textbf{38.32} \\ 
        \specialrule{0em}{-5.25pt}{0pt}
    \end{tabular}
    }
    \caption{\scriptsize Performance Comparison of Pion and Baseline Optimizers on RLVR Tasks. The metric (avg@K) denotes the average accuracy over $K$ generated samples per problem. \textbf{Bold} values indicate the best overall average results.}
    \label{app: tab rlvr}
    \vspace{2mm}
\end{table*}
\newpage
\clearpage
\begin{figure}
    \centering
    \includegraphics[width=\linewidth]{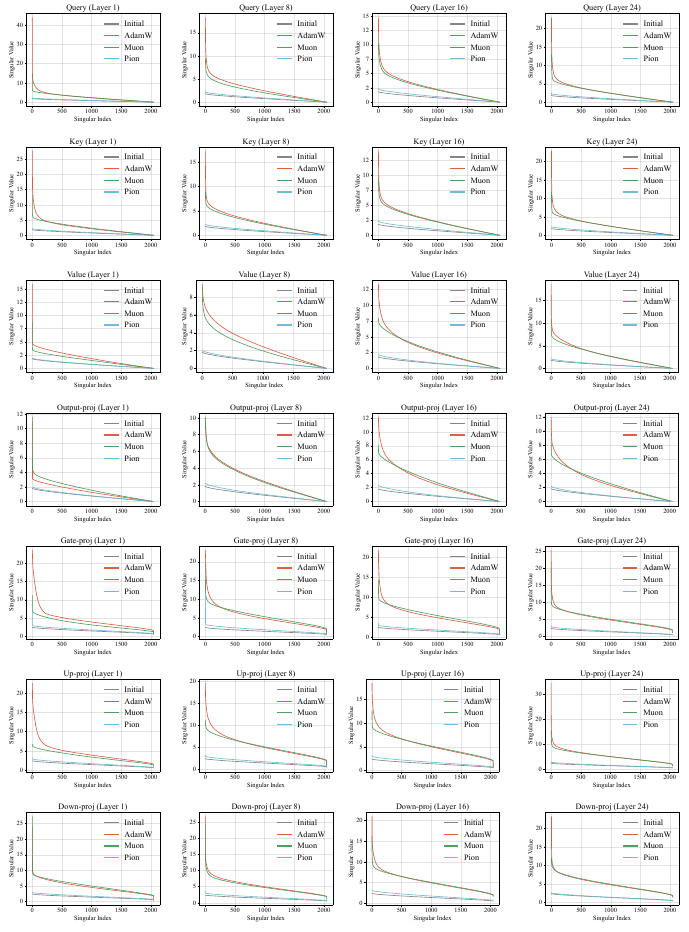}
    \caption{\scriptsize Comparison of final spectrum with the initial spectrum.}
    \label{fig:spectrum_comparison_appendix}
\end{figure}

\newpage
\clearpage
\begin{figure}
    \centering
    \includegraphics[width=\linewidth]{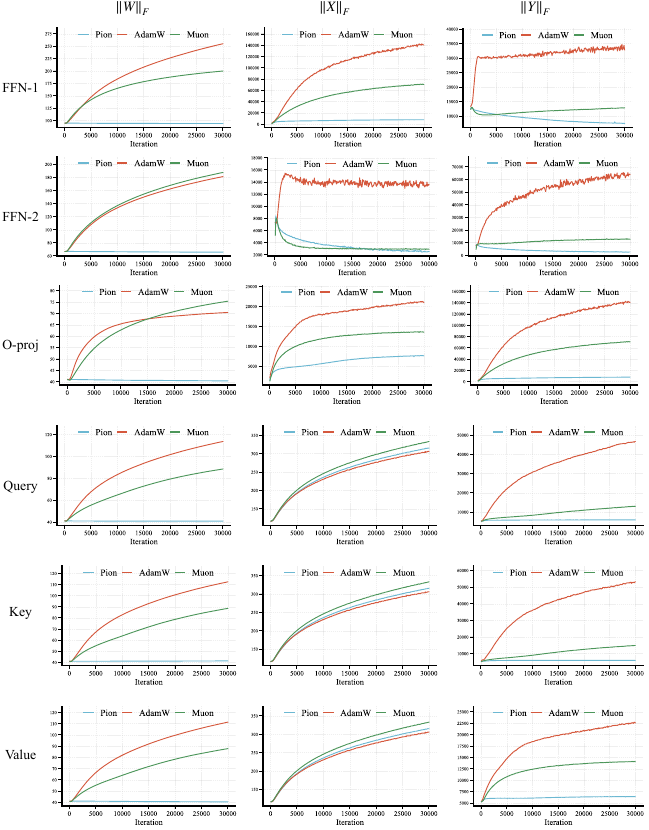}
\caption{\scriptsize Stable indicators for Layer 1. “FFN-1” denotes the combined gate and up projections in the LLaMA-style feed-forward network, where we compute the norm using their overall input and output; “FFN-2” denotes the down projection matrix; “O-proj” denotes the output projection in the attention module; “Query” denotes the query projection; “Key” denotes the key projection; and “Value” denotes the value projection.}
    \label{fig:stable_indicators_layer1_appendix}
\end{figure}
\clearpage

\begin{figure}
    \centering
    \includegraphics[width=\linewidth]{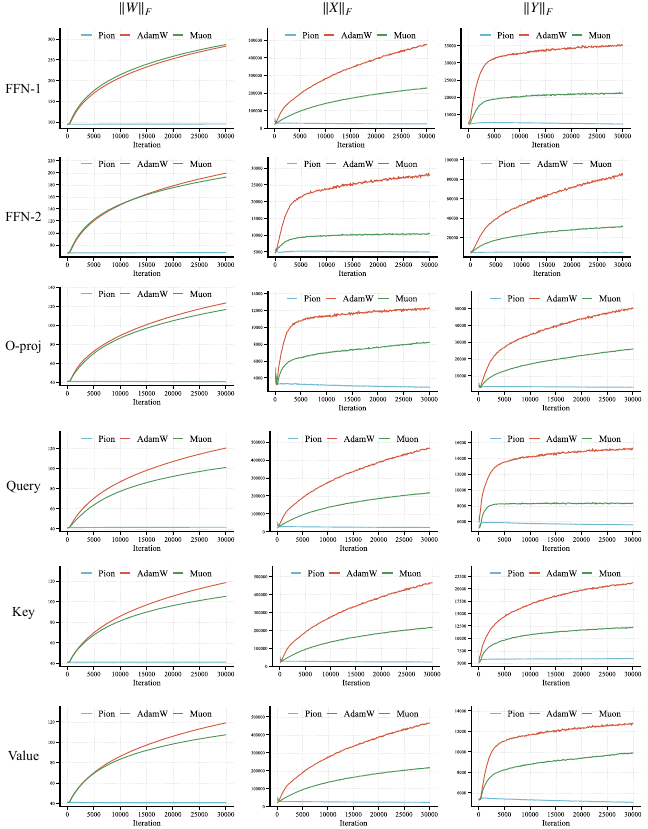}
    \caption{\scriptsize Stable indicators for Layer 12. “FFN-1” denotes the combined gate and up projections in the LLaMA-style feed-forward network, where we compute the norm using their overall input and output; “FFN-2” denotes the down projection matrix; “O-proj” denotes the output projection in the attention module; “Query” denotes the query projection; “Key” denotes the key projection; and “Value” denotes the value projection.}
    \label{fig:stable_indicators_layer12_appendix}
\end{figure}
\clearpage

\begin{figure}
    \centering
    \includegraphics[width=\linewidth]{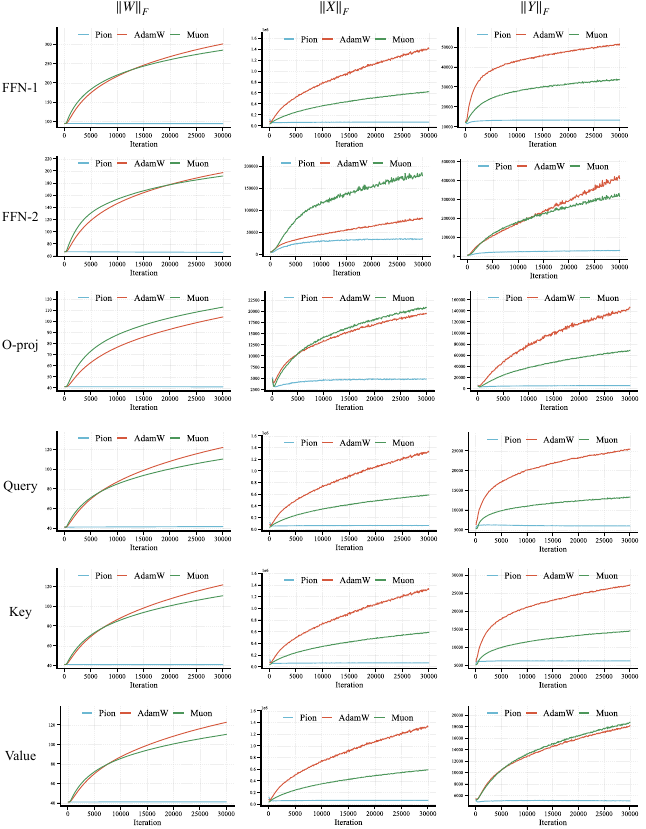}
    \caption{\scriptsize Stable indicators for Layer 24. “FFN-1” denotes the combined gate and up projections in the LLaMA-style feed-forward network, where we compute the norm using their overall input and output; “FFN-2” denotes the down projection matrix; “O-proj” denotes the output projection in the attention module; “Query” denotes the query projection; “Key” denotes the key projection; and “Value” denotes the value projection.}
    \label{fig:stable_indicators_layer24_appendix}
\end{figure}
\clearpage

\newpage
\section{Balanced Rotations in Pion via Matrix Sign Normalization}
\label{app:msign}

{\exhyphenpenalty=10000\relax
We use the 60M LLaMA-based design-exploration setting of Section~\ref{sec:design_exploration_protocol}, with FineWeb-Edu corpus. The sequence length and training budget (256 and 9.6B tokens) are as in Section~\ref{sec:design_exploration_protocol}, and the batch size, learning-rate schedule, and learning-rate grid are as in Appendix~\ref{app:exp_design}. The detailed implementation is in Algorithm~\ref{alg:msign}. Here, $\operatorname{msign}(\cdot)$ denotes the matrix sign. Since msign fixes the scale of each generator, $\alpha=1$ in $\mathcal{E}_2$.

}
\begin{algorithm}[h!]
\caption{Pion Optimizer ($\operatorname{msign}$)}
\label{alg:msign}
\KwIn{Learning rate $\eta$, momentum coefficient $\beta_1$, alternating flag, initial weight matrix $\bm{W}_0$}
\KwOut{Optimized parameter $\bm{W}_t$}
Initialize $\bm{M}_0\leftarrow\bm{0}$\;
Define $\mathcal{E}_2(\bm{A}, \alpha) \leftarrow \bm{I}+\eta\alpha \bm{A}+\frac{1}{2}(\eta\alpha \bm{A})^2$\;
\For{$t=1,2,\ldots$}{
  $\bm{G}_t \leftarrow \nabla_{\bm{W}} f(\bm{W}_{t-1})$\;
  $\bm{M}_t \leftarrow \beta_1 \bm{M}_{t-1} + (1-\beta_1)\bm{G}_t$\;
  $\bm{G}^{\mathrm{in}}_t \leftarrow \bm{W}_{t-1}^{\top}\bm{M}_t - \bm{M}_t^{\top}\bm{W}_{t-1}$\;
  $\bm{G}^{\mathrm{out}}_t \leftarrow \bm{M}_t\bm{W}_{t-1}^{\top} - \bm{W}_{t-1}\bm{M}_t^{\top}$\;
  \eIf{alternative update is used}{
    \eIf{$t$ is even}{
      $\bm{A}^{\mathrm{out}}_t \leftarrow -\operatorname{msign}(\bm{G}^{\mathrm{out}}_t)$\;
      $\bm{W}_t \leftarrow \mathcal{E}_2(\bm{A}^{\mathrm{out}}_t, 1)\bm{W}_{t-1}$\;
    }{
      $\bm{A}^{\mathrm{in}}_t \leftarrow -\operatorname{msign}(\bm{G}^{\mathrm{in}}_t)$\;
      $\bm{W}_t \leftarrow \bm{W}_{t-1}\mathcal{E}_2(\bm{A}^{\mathrm{in}}_t, 1)$\;
    }
  }{
    $\bm{A}^{\mathrm{out}}_t \leftarrow -\operatorname{msign}(\bm{G}^{\mathrm{out}}_t)$\;
    $\bm{A}^{\mathrm{in}}_t \leftarrow -\operatorname{msign}(\bm{G}^{\mathrm{in}}_t)$\;
    $\bm{W}_t \leftarrow \mathcal{E}_2(\bm{A}^{\mathrm{out}}_t, 1)\bm{W}_{t-1}\mathcal{E}_2(\bm{A}^{\mathrm{in}}_t, 1)$\;
  }
}
\Return{$\bm{W}_T$}\;
\end{algorithm}

\end{document}